\documentclass[sts,
	reqno,
]{imsart}

\usepackage{amsthm,amsmath, amssymb, color}
\RequirePackage[colorlinks,citecolor=blue,urlcolor=blue]{hyperref}
\usepackage{xcolor,colortbl}
\usepackage{xr}
\usepackage[square,sort,comma]{natbib}
\RequirePackage{hypernat}

\usepackage{color}
\usepackage{amsfonts, fancybox}
\usepackage{amssymb}
\usepackage[american]{babel}
\usepackage{float}

\usepackage{psfrag,color}
\usepackage{dcolumn}
\usepackage{subcaption}
\usepackage{flafter, bm, dsfont}

\usepackage{multirow}
\usepackage{color}
\usepackage{amsmath}
\usepackage{float}
\usepackage{mathrsfs}
\usepackage{amsfonts}
\usepackage{dsfont}
\usepackage{amssymb}
\usepackage{booktabs}
\usepackage{algpseudocode,algorithm,algorithmicx}
\usepackage{wrapfig}
\usepackage[normalem]{ulem}
\usepackage{siunitx}

\usepackage{amssymb, latexsym}

\usepackage{graphicx}
\usepackage{epstopdf}

\usepackage{boxedminipage}
\usepackage{tkz-graph}
\usetikzlibrary{external}
\newcommand{\supp}{\mathrm{supp}}
\newcommand{\bc}{\begin{center}}
\newcommand{\ec}{\end{center}}
\newcommand{\ba}{\begin{array}}
\newcommand{\ea}{\end{array}}
\newcommand{\be}{\begin{eqnarray}}
\newcommand{\ee}{\end{eqnarray}}
\newcommand{\bel}{\begin{eqnarray}\label}
\newcommand{\eel}{\end{eqnarray}}
\newcommand{\bes}{\begin{eqnarray*}}
\newcommand{\ees}{\end{eqnarray*}}
\newcommand{\bn}{\begin{enumerate}}
\newcommand{\en}{\end{enumerate}}

\definecolor{pink}{cmyk}{0, 1, 0, 0} 
\definecolor{darkgreen}{cmyk}{1,0, 1, 0}

\newtheorem{theorem}{Theorem}

\newtheorem{definition}[theorem]{Definition}

\newcommand{\indic}[1]{\mathds{1}_{#1}}

\usepackage[utf8]{inputenc}
\newcommand{\leadeq}[2][4]{\MoveEqLeft[#1] #2 \nonumber}

\usepackage{stmaryrd}
\usepackage{bbm}
\usepackage{xpunctuate}
\usepackage[abbreviations]{foreign}

\usepackage{mathtools}
\mathtoolsset{showonlyrefs}

\newtheorem{lemmapp}{Lemma}[section]
\DeclareMathOperator{\cl}{cl}
\DeclareMathOperator{\inter}{int}
\DeclareMathOperator{\cost}{cost}

\newcommand{\cC}{\mathcal{C}}
\newcommand{\cD}{\mathcal{D}}

\newcommand{\cF}{\mathcal{F}}

\newcommand{\cH}{\mathcal{H}}

\newcommand{\cK}{\mathcal{K}}

\newcommand{\cM}{\mathcal{M}}
\newcommand{\cN}{\mathcal{N}}

\newcommand{\cP}{\mathcal{P}}
\newcommand{\cQ}{\mathcal{Q}}

\newcommand{\cW}{\mathcal{W}}
\newcommand{\cX}{\mathcal{X}}
\newcommand{\cY}{\mathcal{Y}}

\newcommand{\bone}{\mathbf{1}}

\newcommand{\R}{{\rm I}\kern-0.18em{\rm R}}
\newcommand{\h}{{\rm I}\kern-0.18em{\rm H}}
\newcommand{\K}{{\rm I}\kern-0.18em{\rm K}}
\newcommand{\p}{{\rm I}\kern-0.18em{\rm P}}
\newcommand{\E}{{\rm I}\kern-0.18em{\rm E}}
\newcommand{\Z}{{\rm Z}\kern-0.18em{\rm Z}}
\newcommand{\1}{{\rm 1}\kern-0.24em{\rm I}}
\newcommand{\N}{{\rm I}\kern-0.18em{\rm N}}

\newcommand{\pn}{\p_{\kern-0.25em n}}
\newcommand{\pnm}{\p_{\kern-0.25em n,m}}
\newcommand{\psubm}{\p_{\kern-0.25em m}}
\newcommand{\psubp}{\p_{\kern-0.25em p}}
\newcommand{\cfi}{\cF_{\kern-0.25em \infty}}

\newcommand{\argmin}{\mathop{\mathrm{argmin}}}

\newcommand{\sign}{\mathop{\mathrm{sign}}}
\newcommand{\ud}{\mathrm{d}}

\newcommand{\im}{\mathrm{Im}}

\newtheorem{prop}{Proposition}[section]

\newlength{\minipagewidth}
\begin{document}

\begin{frontmatter}

\title{Statistical Optimal Transport via Factored Couplings}
\runtitle{Factored Optimal Transport}

\author{\fnms{Aden} \snm{Forrow}\ead[label=forrow]{aforrow@mit.edu},
	\fnms{Jan-Christian} \snm{H\"utter}\ead[label=huetter]{huetter@math.mit.edu},
	\fnms{Mor} \snm{Nitzan}\thanksref{tmn}\ead[label = nitzan]{mornitz@gmail.com},
	\fnms{Philippe} \snm{Rigollet}\thanksref{tpr}\ead[label = rigollet]{rigollet@math.mit.edu},
	\fnms{Geoffrey} \snm{Schiebinger}\thanksref{tgw}\ead[label = schiebinger]{geoff@broadinstitute.org},
	\fnms{Jonathan} \snm{Weed}\thanksref{tjw}\ead[label = weed]{jweed@math.mit.edu}
}

\thankstext{tpr}{P.R. is supported in part by NSF grants DMS-1712596 and TRIPODS-1740751 and IIS-1838071, ONR grant N00014-17-1-2147, the Chan Zuckerberg Initiative DAF 2018-182642, and the MIT Skoltech Seed Fund.}
\thankstext{tmn}{M.N. is supported by the James S. McDonnell Foundation, Schmidt Futures, Israel Council for Higher Education, and the John Harvard Distinguished Science Fellows Program.}
\thankstext{tjw}{J.W. is supported in part by the Josephine de Karman fellowship.}
\thankstext{tgw}{G.S. is supported by a Career Award at the Scientific Interface from the Burroughs Welcome Fund and by the Klarman Cell Observatory.}

\begin{abstract}
We propose a new method to estimate Wasserstein distances and optimal transport plans between two probability distributions from samples in high dimension. Unlike plug-in rules that simply replace the true distributions by their empirical counterparts, our method promotes couplings with low \emph{transport rank}, a new structural assumption that is similar to the nonnegative rank of a matrix.
Regularizing based on this assumption leads to drastic improvements on high-dimensional data for various tasks, including domain adaptation in single-cell RNA sequencing data. These findings are supported by a theoretical analysis that indicates that the transport rank is key in overcoming the curse of dimensionality inherent to data-driven optimal transport.
\end{abstract}

\end{frontmatter}

\section{INTRODUCTION}\label{sec:intro}

Optimal transport (OT) was born from a simple question phrased by Gaspard Monge in the eighteenth century~\citep{Mon81} and has since flourished into a rich mathematical theory two centuries later~\citep{villani-2003, Vil09}. Recently, OT and more specifically Wasserstein distances, which include the so-called \emph{earth mover's distance}~\citep{rubner-2000} as a special example, have proven valuable for varied tasks in machine learning~\citep{bassetti2006minimum,cuturi2013sinkhorn,cuturi-2014,solomon2014wasserstein,frogner-2015,srivastava2015wasp,genevay2016stochastic,gao2016distributionally,rolet-2016,genevay2017sinkhorn,RigWee18a,RigWee18b}, computer graphics~\citep{bonneel-2011,degoes-2012,solomon2014earth,solomon-2015,bonneel-2016}, geometric processing~\citep{degoes-2011,solomon2013dirichlet}, image processing~\citep{gramfort-2015,rabin-2015}, and document retrieval~\citep{ma-2014,kusner-2015}. These recent developments have been supported by breakneck advances in computational optimal transport in the last few years that allow the approximation of these distances in near linear time~\citep{cuturi2013sinkhorn, AltWeeRig17}.

In these examples, Wasserstein distances and transport plans are estimated from data. Yet, the understanding of \emph{statistical} aspects of OT is still in its infancy. In particular, current methodological advances focus on computational benefits but often overlook statistical regularization to address stability in the presence of sampling noise. Known theoretical results show that vanilla optimal transport applied to sampled data suffers from the curse of dimensionality~\citep{Dud69,DobYuk95,WeeBac17} and the need for principled regularization techniques is acute in order to scale optimal transport to high-dimensional problems, such as those arising in genomics, for example.

At the heart of OT is the computation of Wasserstein distances, which consists of an optimization problem over the infinite dimensional set of \emph{couplings} between probability distributions. (See~\eqref{eq:wp} for a formal definition.) Estimation in this context is therefore nonparametric in nature and this is precisely the source of the curse of dimensionality. To overcome this limitation, and following a major trend in high-dimensional statistics~\citep{CanPla10,LiuLinYu10,MarUse12}, we propose to impose low ``rank" structure on the couplings. Interestingly, this technique can be implemented efficiently via \emph{Wasserstein barycenters}~\citep{AguCar11,CutDou14} with finite support.

We illustrate the performance of this new procedure for a truly high-dimensional problem arising in single-cell RNA sequencing data, where ad-hoc methods for domain adaptation have recently been proposed to couple datasets collected in different labs and with different protocols~\citep{haghverdi2017correcting}, and even across
species~\citep{butler2018integrating}. Despite a relatively successful application of OT-based methods in this context~\citep{SchShuTab17}, the very high-dimensional and noisy nature of this data calls for robust statistical methods. We show in this paper that our proposed method does lead to improved results for this application.

This paper is organized as follows. We begin by reviewing optimal transport in \S\ref{sec:OT},
%
and we provide an overview of our results in \S\ref{sec:overview}. Next, we introduce our estimator in \S\ref{sec:hub}. This is a new estimator for the Wasserstein distance between two probability measures that is statistically more stable than the naive plug-in estimator that has traditionally been used. This stability guarantee is not only
backed by the theoretical results of \S\ref{sec:theory}, but also observed in numerical experiments in practice \S\ref{sec:experiments}.


\noindent{\bf Notation.} We denote by $\|\cdot\|$ the Euclidean norm over $\R^d$. For any $x \in \R^d$, let $\delta_x$ denote the Dirac measure centered at $x$. For any two real numbers $a$ and $b$, we denote their minimum by $a\wedge b$. For any two sequences $u_n,v_n$, we write $u_n\lesssim v_n$ when there exists a constant $C>0$ such that $u_n \le C v_n$ for all $n$. If $u_n\lesssim v_n$ and $v_n\lesssim u_n$, we write $u_n \asymp v_n$. We denote by $\bone_n$ the all-ones vector of $\R^n$, and by $e_i$ the $i$th standard vector in $\R^n$. Moreover, we denote by $\odot$ and $\oslash$ element-wise multiplication and division of vectors, respectively.

For any map $f:\R^d \to \R^d$ and measure $\mu$ on $\R^d$, let $f_{\#}\mu$ denote the pushforward measure of $\mu$ through $f$ defined for any Borel set $A$ by $f_{\#}\mu(A)=\mu \big(f^{-1}(A)\big)$, where $f^{-1}(A)=\{x\in \R^d\,: f(x)\in A\}$. Given a measure $\mu$, we denote its support by $\supp(\mu)$.

\section{BACKGROUND ON OPTIMAL TRANSPORT}
\label{sec:OT}
In this section, we gather the necessary background on optimal transport. We refer the reader to recent books~\citep{San15,villani-2003,Vil09} for more details.

\paragraph{Wasserstein distance} Given two probability measures $P_0$ and $P_1$ on $\R^d$, let $\Gamma(P_0,P_1)$ denote the set of \emph{couplings} between $P_0$ and $P_1$, that is, the set of joint distributions with marginals $P_0$ and $P_1$ respectively so that $\gamma \in \Gamma(P_0,P_1)$ iff $\gamma(U\times \R^d)=P_0(U)$ and $\gamma(\R^d\times V) = P_1(V)$ for all measurable $U, V \in \R^d$.

The \emph{$2$-Wasserstein distance}\footnote{In this paper we omit the prefix ``2-'' for brevity.} between two probability measures $P_0$ and $P_1$ is defined as
\begin{equation}\label{eq:wp}
W_2(P_0,P_1) :=\hspace{-1.em}\inf_{\gamma\in\Gamma(P_0,P_1)} \sqrt{\int_{\R^d\times \R^d}\hspace{-2em} \|x-y\|^2\, \ud\gamma(x,y)}\,.
\end{equation}

Under regularity conditions, for example if both $P_0$ and $P_1$ are absolutely continuous with respect to the Lebesgue measure, it can be shown the infimum in~\eqref{eq:wp} is attained at a unique coupling $\gamma^*$. Moreover $\gamma^*$ is a deterministic coupling: it is supported on a set of the form $\{(x, T(x))\,:\, x \in \supp(P_0)\}$. In this case, we call $T$ a transport \emph{map}.
In general, however, $\gamma^*$ is unique but for any $x_0 \in \supp(P_0)$, the support of $\gamma^*(x_0, \cdot)$ may not reduce to a single point, in which case, the map $x \mapsto \gamma^*(x, \cdot)$ is called a transport \emph{plan}.

\paragraph{Wasserstein space} The space of probability measures with finite $2$nd moment equipped with the metric $W_2$ is called Wasserstein space and denoted by $\cW_2$. It can be shown that $\cW_2$ is a geodesic space: given two probability measures $P_0, P_1 \in \cW_2$, the constant speed geodesic connecting $P_0$ and $P_1$ is the curve $\{P_t\}_{t \in [0,1]}$ defined as follows. Let $\gamma^*$ be the optimal coupling defined as the solution of~\eqref{eq:wp}, and for $t \in [0,1]$ let $\pi_t: \R^d \times \R^d \to \R$ be defined as $\pi_t(x,y)=(1-t)x+ty$, then $P_t=(\pi_t)_{\#}\gamma^*$.  We then call $P_{1/2}$ the geodesic midpoint of $P_0$ and $P_1$. It plays the role of an average in Wasserstein space, which, unlike the mixture $(P_0 + P_1)/2$, takes the geometry of $\R^d$ into account.


\smallskip
\paragraph{$k$-Wasserstein barycenters}
The now-popular notion of Wasserstein barycenters (WB) was introduced by~\cite{AguCar11} as a generalization
of the geodesic midpoint~$P_{1/2}$ to more than two measures.
In its original form, a WB can be any probability measure on $\R^d$, but algorithmic considerations led Cuturi and Doucet~\citep{CutDou14} to restrict the support of a WB to a finite set of size $k$. Let $\cD_k$ denote the set of probability distributions supported on $k$ points:
$$
\cD_k=\left\{\sum_{j=1}^k\alpha_j\delta_{x_j}\,:\,\alpha_j\ge 0, \sum_{j=1}^k\alpha_j=1, x_j \in \R^d \right\}.
$$
For a given integer $k$, the \emph{$k$-Wasserstein Barycenter} $\bar P$ between $N$ probability measures $P_0, \ldots P_N$ on $\R^d$ is defined by
\begin{equation}
\label{eq:WB}
\bar P = \argmin_{P \in \cD_k}\sum_{j=1}^N W_2^2(P, P^{(j)})\,.
\end{equation}
In general~\eqref{eq:WB} is not a convex problem but fast numerical heuristics have demonstrated good performance in practice~\citep{CutDou14,CutPey16,BenCarCut15,staib2017parallel,claici2018stochastic}. Interestingly, Theorem~\ref{thm:empirical_process} below indicates that the extra constraint $P \in \cD_k$ is also key to statistical stability.


\section{RESULTS OVERVIEW}
\label{sec:overview}

Ultimately, in all the data-driven applications cited above, Wasserstein distances must be estimated from data. While this is arguably the most fundamental primitive of all OT based machine learning, the statistical aspects of this question are often overlooked at the expense of computational ones. We argue that standard estimators of both $W_2(P_0,P_1)$ and its associated optimal transport plan suffer from statistical instability. The main contribution of this paper is to overcome this limitation  by injecting statistical regularization.

\paragraph{Previous work}
Let $X \sim P_0$ and $Y\sim P_1$ and let $X_1, \ldots, X_n$ (resp. $Y_1, \ldots, Y_n$) be independent copies of $X$ (resp. $Y$).\footnote{Extensions to the case where the two sample sizes differ are straightforward but do not enlighten our discussion.}  We call $\cX=\{X_1, \ldots, X_n\}$ and $\cY=(Y_1, \ldots, Y_n)$ the \emph{source} and \emph{target} datasets respectively. Define the corresponding empirical measures:
\begin{equation}
\label{eq:empirical_measures}
\hat P_0=\frac{1}n\sum_{i=1}^n \delta_{X_i}\,, \qquad \hat P_1=\frac{1}n\sum_{i=1}^n \delta_{Y_i}\,.
\end{equation}
Perhaps the most natural estimator for $W_2(P_0,P_1)$, and certainly the one most employed and studied, is the \emph{plug-in} estimator $W_2(\hat P_0,\hat P_1)$. A natural question is to determine the accuracy of this estimator. This question was partially addressed by Sommerfeld and Munk~\citep{SomMun17}, where the rate at which $\Delta_{n}:=|W_2(\hat P_0,\hat P_1)-W_2(P_0,P_1)|$ vanishes is established. They show that $\Delta_n \asymp n^{-1/2} $ if $P_0 \neq P_1$ and $\Delta_n \asymp n^{-1/4} $ if $P_0=P_1$. Unfortunately, these rates are only valid when $P_0$ and $P_1$ have finite support. Moreover, the plug-in estimator for distributions $\R^d$ has been known to suffer from the curse of dimensionality at least since the work of Dudley~\citep{Dud69}. More specifically, in this case, $\Delta_{n}\asymp n^{-1/d}$ when $d \ge 3$~\citep{DobYuk95}. One of the main goals of this paper is to provide an alternative to the naive plug-in estimator by regularizing the optimal transport problem~\eqref{eq:wp}. Explicit regularization for optimal transport problems was previously introduced by Cuturi~\citep{cuturi2013sinkhorn} who adds an entropic penalty to the objective in~\eqref{eq:wp} primarily driven by algorithmic motivations. While entropic OT was recently shown~\citep{RigWee18b} to also provide statistical regularization, that result indicates that entropic OT does not alleviate the curse of dimensionality coming from sampling noise, but rather addresses the presence of additional measurement noise. 

Closer to our setup are~\cite{CouFlaTui14} and \cite{FerPapPey14}; both consider sparsity-inducing structural penalties that are relevant for domain adaptation and computer graphics, respectively. While the general framework of Tikhonov-type regularization for optimal transport problems is likely to bear fruit in specific applications, we propose
a new general-purpose \emph{structural} regularization method, based on a new notion of complexity for joint probability measures.


\paragraph{Our contribution}
The core contribution of this paper is to construct an estimator of the Wasserstein distance between distributions that is more stable and accurate under sampling noise. We do so by defining a new regularizer for couplings, which we call the \emph{transport rank}. As a byproduct, our estimator also yields an estimator of the optimal coupling in~\eqref{eq:wp} that can in turn be used in domain adaptation where optimal transport has recently been employed~\citep{CouFlaTui14,CouFlaTui17}.


To achieve this goal, we leverage insights from a popular technique known as nonnegative matrix factorization (NMF)~\citep{PaaPen94,LeeSeu01} which has been successfully applied in various forms to many fields, including text analysis~\citep{ShaBerPau06}, computer vision~\citep{ShaHaz05}, and bioinformatics~\citep{gao2005improving}. Like its cousin factor analysis, it postulates the existence of low-dimensional latent variables that govern the high-dimensional data-generating process under study. 

In the context of  optimal transport, we consider couplings $\gamma \in \Gamma(P_0,P_1)$ such that whenever $(X,Y)\sim \gamma$, there exits a latent variable $Z$ with \emph{finite support} such that $X$ and $Y$ are conditionally independent given $Z$. To see the analogy with NMF, one may view a coupling $\gamma$ as a doubly stochastic matrix whose rows and columns are indexed by $\R^d$. We consider couplings such that this matrix can be written as the product~$AB$ where $A$ and $B^\top$ are matrices whose rows are indexed by $\R^d$ and columns are indexed by $\{1, \ldots k\}$. In that case, we call $k$ the \emph{transport rank} of $\gamma$. We now formally define these notions.
\begin{definition}
Given $\gamma \in \Gamma(P_0, P_1)$, the \emph{transport rank} of $\gamma$ is the smallest integer $k$ such that $\gamma$ can be written
\begin{equation}\label{eq:factored_coupling}
\gamma = \sum_{j=1}^k \lambda_j (Q^0_j \otimes Q^1_j)\,,
\end{equation}
where the $Q^0_j$'s and $Q^1_j$'s are probability measures on $\R^d$, $\lambda_j \geq 0$ for $j = 1, \dots, k$, and where $Q^0_j \otimes Q^1_j$ indicates the (independent) product distribution. We denote the set of couplings between $P_0$ and $P_1$ with transport rank at most $k$ by $\Gamma_k(P_0, P_1)$.
\end{definition}

When $P_0$ and $P_1$ are finitely supported,
the transport rank of $\gamma \in \Gamma(P_0, P_1)$ coincides with the nonnegative rank~\citep{Yan91,CohRot93} of $\gamma$ viewed as a matrix.
By analogy with a nonnegative factorization of a matrix, we call a coupling written as a sum as in~\eqref{eq:factored_coupling} a \emph{factored coupling}.
Using the transport rank as a regularizer therefore promotes simple couplings, i.e., those possessing a low-rank ``factorization.''
To implement this regularization, we show that it can be constructed via $k$-Wasserstein barycenters, for which efficient implementation is readily available.


As an example of our technique, we show in \S\ref{sec:experiments} that this approach can be used to obtain better results on \emph{domain adaptation} a.k.a \emph{transductive learning}, a strategy in semi-supervised learning to transfer label information from a source dataset to a target dataset. Notably, while regularized optimal transport has proved to be an effective tool for \emph{supervised} domain adaptation where label information is used to build an explicit Tikhonov regularization~\citep{CouFlaTui14}, our approach is entirely unsupervised, in the spirit of~\cite{gong2012geodesic} where unlabeled datasets are matched and then labels are transported from the source to the target. We argue that both approaches, supervised and unsupervised, have their own merits but the unsupervised approach is more versatile and calibrated with our biological inquiry regarding single cell data integration.

\section{REGULARIZATION VIA FACTORED COUPLINGS}
\label{sec:hub}

\algnewcommand\algorithmicinput{\textbf{Input:}}
\algnewcommand\Input{\item[\algorithmicinput]}
\algnewcommand\algorithmicoutput{\textbf{Output:}}
\algnewcommand\Output{\item[\algorithmicoutput]}

To estimate the Wasserstein distance between $P_0$ and $P_1$, we find a low-rank factored coupling between the empirical distributions. As we show in \S\ref{sec:theory}, the bias induced by this regularizer provides significant statistical benefits. Our procedure is based on an intuitive principle: optimal couplings arising in practice can be well approximated by assuming the distributions have a small number of pieces moving nearly independently. For example, if distributions represent populations of cells, this assumption amounts to assuming that there are a small number of cell ``types,'' each subject to different forces.

Before introducing our estimator, we note that a factored coupling induces coupled partitions of the source and target distributions. These clusterings are ``soft'' in the sense that they may include fractional points.

\begin{definition} Given $\lambda \in [0,1]$, a \emph{soft cluster} 
of a probability measure $P$ is a sub-probability measure $C$ of total mass $\lambda$ such that $0 \leq C \leq P$ as measures.
The \emph{centroid} of $C$ is defined by $\mu(C) = \frac{1}{\lambda} \int x \, \mathrm{d}C(x)$.
We say that a collection $C_1, \dots, C_k$ of soft clusters of $P$ is a \emph{partition} of $P$ if $C_1 + \dots + C_k = P$.
\end{definition}

The following fact is immediate.
\begin{prop}\label{prop:factored_to_cluster}
If $\gamma = \sum_{j=1}^k \lambda_j (Q^0_j \otimes Q^1_j)$ is a factored coupling in $\Gamma_k(P_0, P_1)$, then $\{\lambda_1 Q^0_1, \dots, \lambda_k Q^0_k\}$ and $\{\lambda_1 Q^1_1, \dots, \lambda_k Q^1_k\}$ are partitions of $P_0$ and $P_1$, respectively.
\end{prop}


We now give a simple characterization of the ``cost'' of a factored coupling.
\begin{prop}\label{prop:factored_bias_variance}
Let $\gamma \in \Gamma_k(P_0, P_1)$ and let $C_1^0, \dots, C_k^0$ and $C_1^1, \dots, C_k^1$ be the induced partitions of $P_0$ and $P_1$, with $C_j^0(\R^d) = C_j^1(\R^d) = \lambda_j$ for $j = 1, \dots k$.
Then
\begin{align*}
\int \|x - y\|^2 \,\mathrm{d}&\gamma(x, y) =  \sum_{j=1}^k \Big(\lambda_j \|\mu(C_j^0) - \mu(C_j^1)\|^2\\
&+\sum_{l\in \{0,1\}}\int \|x - \mu(C_j^l)\|^2 \, \mathrm{d}C_j^l(x) \Big)
\end{align*}
\end{prop}
The sum over $l$ in the above display contains intra-cluster variance terms similar to the $k$-means objective, while the first term is a transport term reflecting the cost of transporting the partition of $P_0$ to the partition of $P_1$.
Since our goal is to estimate the transport distance, we focus on the first term. This motivates the following definition.
\begin{definition}
The cost of a factored transport $\gamma \in \Gamma_k(P_0, P_1)$ is $$\cost(\gamma) := \sum_{j=1}^k \lambda_j \|\mu(C_j^0) - \mu(C_j^1)\|^2\,$$ where $\{C_j^0\}_{j=1}^k$ and $\{C_j^1\}_{j=1}^k$ are the partitions of $P_0$ and $P_1$ induced by $\gamma$, with $C_j^0(\R^d) = C_j^1(\R^d) = \lambda_j$ for $j = 1, \dots, k$.
\end{definition}

Given empirical distributions $\hat P_0$ and $\hat P_1$, the (unregularized) optimal coupling between $\hat P_0$ and $\hat P_1$, defined as
$$
\argmin_{\gamma \in \Gamma(\hat P_0, \hat P_1)} \int \|x - y\|^2 \mathrm{d}\gamma(x, y)\,,
$$
is highly non-robust with respect to sampling noise.
This motivates considering instead the regularized version 
\begin{equation}\label{eq:factored_ot}
\argmin_{\gamma \in \Gamma_k(\hat P_0, \hat P_1)} \int \|x - y\|^2 \mathrm{d}\gamma(x, y)\,,
\end{equation}
where $k \geq 1$ is a regularization parameter.
Whereas fast solvers are available for the unregularized problem~\citep{AltWeeRig17}, it is not clear how to find a solution to~\eqref{eq:factored_ot} by similar means. While alternating minimization approaches similar to heuristics for nonnegative matrix factorization are possible~\citep{LeeSeu01,AroraNMF}, we adopt a different approach which has the virtue of connecting~\eqref{eq:factored_ot} to $k$-Wasserstein barycenters.

Following~\cite{CutDou14}, define the $k$-Wasserstein barycenter of $\hat P_0$ and $\hat P_1$ by
\begin{equation}\label{eq:hubs}
H = \argmin_{P \in \cD_k} \,\,\left\{ W_2^2(P, \hat P_0) + W_2^2(P, \hat P_1)\right\}\,.
\end{equation}
As noted above, while this objective is not convex, efficient procedures have been shown to work well in practice.

Strikingly, the $k$-Wasserstein barycenter of $\hat P_0$ and $\hat P_1$ implements a slight variant of~\eqref{eq:factored_ot}.
Given a feasible $P \in \cD_k$ in~\eqref{eq:hubs}, we first note that it induces 
a factored coupling in $\Gamma_k(\hat P_0, \hat P_1)$.
Indeed, denote by $\gamma_0$ and $\gamma_1$ the optimal couplings between $\hat P_0$ and $P$ and $P$ and $\hat P_1$, respectively.
Write $z_1, \dots, z_j$ for the support of $P$.
We can then decompose these couplings as follows:
\begin{align*}
\gamma_0  = \sum_{j=1}^k \gamma_0(\cdot \mid z_j) H(z_j), \quad
\gamma_1  = \sum_{j=1}^k \gamma_1(\cdot \mid z_j) H(z_j)
\end{align*}
Then for any Borel sets $A,B \subset \R^d$,
$$
\gamma_P(A \times B) := \sum_{j=1}^k P(z_j) \gamma_0(A | z_j) \gamma_1(B | z_j) \in \Gamma_k(\hat P_0, \hat P_1)
$$
and by the considerations above, this factored transport induces coupled partitions $C_1^0, \dots, C_k^0$ and $C_1^1, \dots, C_k^1$ of $\hat P_0$ and $\hat P_1$ respectively. We call the points $z_1, \dots, z_j$ ``hubs.''

The following proposition gives optimality conditions for $H$ in terms of this partition.

\begin{prop}\label{prop:hubs_to_factored}
The partitions $C_1^0, \dots, C_k^0$ and $C_1^1, \dots, C_k^1$ induced by the solution $H$ of~\eqref{eq:hubs} are the minimizers of
$$
\sum_{j=1}^k \Big(\frac{\lambda_j}{2} \|\mu(C_j^0) - \mu(C_j^1)\|^2+\sum_{l =0}^1\int \|x - \mu(C_j^l)\|^2 \, \mathrm{d}C_j^l(x) \Big)
$$
where $\lambda_j=\mu(C_j^0)=\mu(C_j^1)$ and the minimum is taken over all partitions of $\hat P_0$ and $\hat P_1$ induced by feasible $P \in \cD_k$.
\end{prop}
Comparing this result with Proposition~\ref{prop:factored_bias_variance}, we see that this objective agrees with the objective of~\eqref{eq:factored_ot} up to a multiplicative factor of $1/2$ in the transport term.

We therefore view~\eqref{eq:hubs} as a algorithmically tractable proxy for~\eqref{eq:factored_ot}. Hence, we propose the following estimator $\hat W$ of the squared Wasserstein distance:
\begin{equation}\label{eq:estimator}
\hat W := \cost(\gamma_H)\,, \quad \quad \text{where $H$ solves~\eqref{eq:hubs}}\,.
\end{equation}
We can also use $\gamma_H$ to construct an estimated transport map $\hat T$ on the points $X_1, \dots, X_n \in \supp(\hat P_0)$ by setting
$$
\hat T(X_i) = X_i + \frac{1}{\sum_{j=1}^k C_j^0(X_i)} \sum_{j=1}^k C_j^0(X_i) (\mu(C_j^1) - \mu(C_j^0))\,.
$$
Moreover, the quantity $\hat T_\sharp \hat P_0$ provides a stable estimate of the target distribution, which is particularly useful in domain adaptation.

Our core algorithmic technique involves computing a $k$-Wasserstein Barycenter as in~\eqref{eq:WB}. This problem is non-convex in the variables \( \cM \) and \( (\gamma_0, \gamma_1) \), but it is separately convex in each of the two.
Therefore, it admits an alternating minimization procedure similar to Lloyd's algorithm for $k$-means~\citep{Llo82}, which we give in Algorithm \ref{alg:kbary}.
The update with respect to the hubs \( \cH = \{z_1, \dots, z_k\} \), given plans \( \gamma_0 \) and \( \gamma_1 \) can be seen to be a quadratic optimization problem, with the explicit solution
\begin{equation*}
	z_j = \frac{\sum_{i = 1}^{n} \gamma_0(z_j, X_i) X_i + \sum_{i=1}^{n} \gamma_1(z_j, Y_i) Y_i}{\sum_{i=1}^{n} \gamma_0(z_j, X_i) + \sum_{i=1}^{n} \gamma_1(z_j, Y_i)},
\end{equation*}
leading to Algorithm \ref{alg:up-points}.

In order to solve for the optimal \( (\gamma_0, \gamma_1) \) given a value for the hubs \( \cH = \{z_1, \dots, z_k\} \) we add the following entropic regularization terms~\citep{cuturi2013sinkhorn} 
to the objective function \eqref{eq:hubs}:
\begin{equation*}
	-\varepsilon \sum_{i, j} (\gamma_0)_{j, i} \log((\gamma_0)_{j, i}) - \varepsilon \sum_{i, j} (\gamma_1)_{j, i} \log((\gamma_1)_{j, i}),
\end{equation*}
 where \( \varepsilon > 0 \) is a small regularization parameter.
This turns the optimization over \( (\gamma_0, \gamma_1) \) into a projection problem with respect to the Kullback-Leibler divergence, which can be solved by a type of Sinkhorn iteration, see \cite{BenCarCut15} and Algorithm \ref{alg:up-plans}. For small \( \varepsilon \), this will yield a good approximation to the optimal value of the original problem,
but the Sinkhorn iterations become increasingly unstable.
We employ a numerical stabilization strategy due to \cite{Sch16} and \cite{ChiPeySch16}.
Also, an initialization for the hubs is needed, for which we suggest using a $k$-means clustering of either $\cX$ or $\cY$.

\begin{algorithm}[H]
	\caption{{\sc FactoredOT}}
	\label{alg:kbary}
	\begin{algorithmic}
		\Input Sampled points $\cX, \cY$, parameter $ \varepsilon > 0$
		\Output Hubs $\cM$, transport plans  $\gamma_0, \gamma_1$
		\Function{FactoredOT}{$\cX, \cY, \varepsilon$}
		\State Initialize $\cM$, e.g $\cM \gets$ \Call{KMeans}{$\cX$}
		\While{not converged}
		\State \( (\gamma_0, \gamma_1) \gets \Call{UpdatePlans}{\cX, \cY, \cM}\)
		\State \( \cM \gets \Call{UpdateHubs}{\cX, \cY, \gamma_0, \gamma_1}\)
		\EndWhile
		\State \Return \( (\cM, \gamma_0, \gamma_1) \)
		\EndFunction
	\end{algorithmic}
\end{algorithm}

\begin{algorithm}[H]
	\caption{{\sc UpdateHubs}}
	\label{alg:up-points}
	\begin{algorithmic}
		\Function{UpdateHubs}{$ \cX, \cY, \gamma_0, \gamma_1 $}
		\For{$j = 1, \dots, k$}
        \State $p^{(0)}_{i,j}=\gamma_0(z_j, X_i);\ p^{(1)}_{i,j} =\gamma_1(z_j, Y_i)$
		\State $ z_j \gets \frac{ \sum_{i = 1}^{n} \{p^{(0)}_{i,j} X_i +  p^{(0)}_{i,j} Y_i\} }{\sum_{i=1}^{n} \{p^{(0)}_{i,j} + p^{(1)}_{i,j}\}}$
		\EndFor
		\EndFunction
	\end{algorithmic}
\end{algorithm}

\begin{algorithm}[H]
	\caption{{\sc UpdatePlans}}
	\label{alg:up-plans}
	\begin{algorithmic}
		\Require Points \( \cX, \cY \), hubs \( \cM \), parameter \( \varepsilon > 0\)
		\Function{UpdatePlans}{$ \cX, \cY, \cM, \varepsilon $}
		\State $u_0 = u_1 = \bone_k, \; v_0 = v_1 = \bone_n$
		\State $(\xi_0)_{j, i} = \exp( \| z_j - X_i \|_2^2/\varepsilon)$
		\State $(\xi_1)_{j, i} = \exp( \| z_j - Y_i \|_2^2/\varepsilon)$
		\While{not converged}
		\State $ v_0 = \frac{1}{n}\bone_n \oslash (\xi_0^\top u_0)\;\   v_1 = \frac{1}{n}\bone_n \oslash (\xi_1^\top u_1)$
		\State $ w = (u_0 \odot (\xi_0 v_0))^{1/2} \odot (u_1 \odot (\xi_1 v_1))^{1/2}$
        \State $u_0 = w \oslash (\xi_0 v_0); \  u_1 = w \oslash (\xi_1 v_1)$
		\EndWhile\\
		\Return $ (\operatorname{diag}(u_0) \xi_0 \operatorname{diag}(v_0), \operatorname{diag}(u_1) \xi_1 \operatorname{diag}(v_1)) $
		\EndFunction
	\end{algorithmic}
\end{algorithm}

\section{THEORY}\label{sec:theory}
In this section, we give theoretical evidence that the use of factored transports makes our procedure more robust.
In particular, we show that it can overcome the ``curse of dimensionality'' generally inherent to the use of Wasserstein distances on empirical data.

To make this claim precise, we show that the objective function in~\eqref{eq:hubs} is robust to sampling noise.
This result establishes that despite the fact that the unregularized quantity $W_2^2(\hat P_0, \hat P_1)$ approaches $W_2^2(P_0, P_1)$ very slowly, the empirical objective
in~\eqref{eq:hubs}
approaches the population objective
uniformly at the \emph{parametric} rate, thus significantly improving the dependence on the dimension.
Via the connection between~\eqref{eq:hubs} and factored couplings established in Proposition~\ref{prop:hubs_to_factored}, this result implies that regularizing by transport rank yields significant statistical benefits.

\begin{theorem}\label{thm:empirical_process}
Let $P$ be a measure on $\R^d$ supported on the unit ball, and denote by $\hat P$ an empirical distribution comprising $n$ i.i.d.\ samples from $P$.
Then with probability at least $1-\delta$,
\begin{equation}
\sup_{\rho \in \cD_k} |W^2_2(\rho, \hat P) - W^2_2(\rho, P)| \lesssim \sqrt{\frac{k^3 d \log k + \log(1/\delta)}{n}}\,.
\end{equation}
\end{theorem}
A simple rescaling argument implies that this $n^{-1/2}$ rate holds for all compactly supported measures.

This result complements and generalizes known results from the literature on $k$-means quantization~\citep{Pol82,RakCap06,MauPon10}. Indeed, as noted above, the $k$-means objective is a special case of a squared $W_2$ distance to a discrete measure~\citep{Pol82}. Theorem~\ref{thm:empirical_process} therefore recovers the $n^{-1/2}$ rate for the generalization error of the $k$-means objective; however, our result applies more broadly to \emph{any} measure $\rho$ with small support. Though the parametric $n^{-1/2}$ rate is optimal, we do not know whether the dependence on $k$ or $d$ in Theorem~\ref{thm:empirical_process} can be improved. We discuss the connection between our work and existing results on $k$-means clustering in the supplement.

Finally note that while this analysis is a strong indication of the stability of our procedure, it does not provide explicit rates of convergence for $\hat W$ defined in~\eqref{eq:estimator}. This question requires a structural description of the optimal coupling between $P_0$ and $P_1$ and is beyond the scope of the present paper.

\begin{figure*}[t]
\begin{center}
	\begin{subfigure}[t]{0.3\textwidth}
		\includegraphics[width=\textwidth]{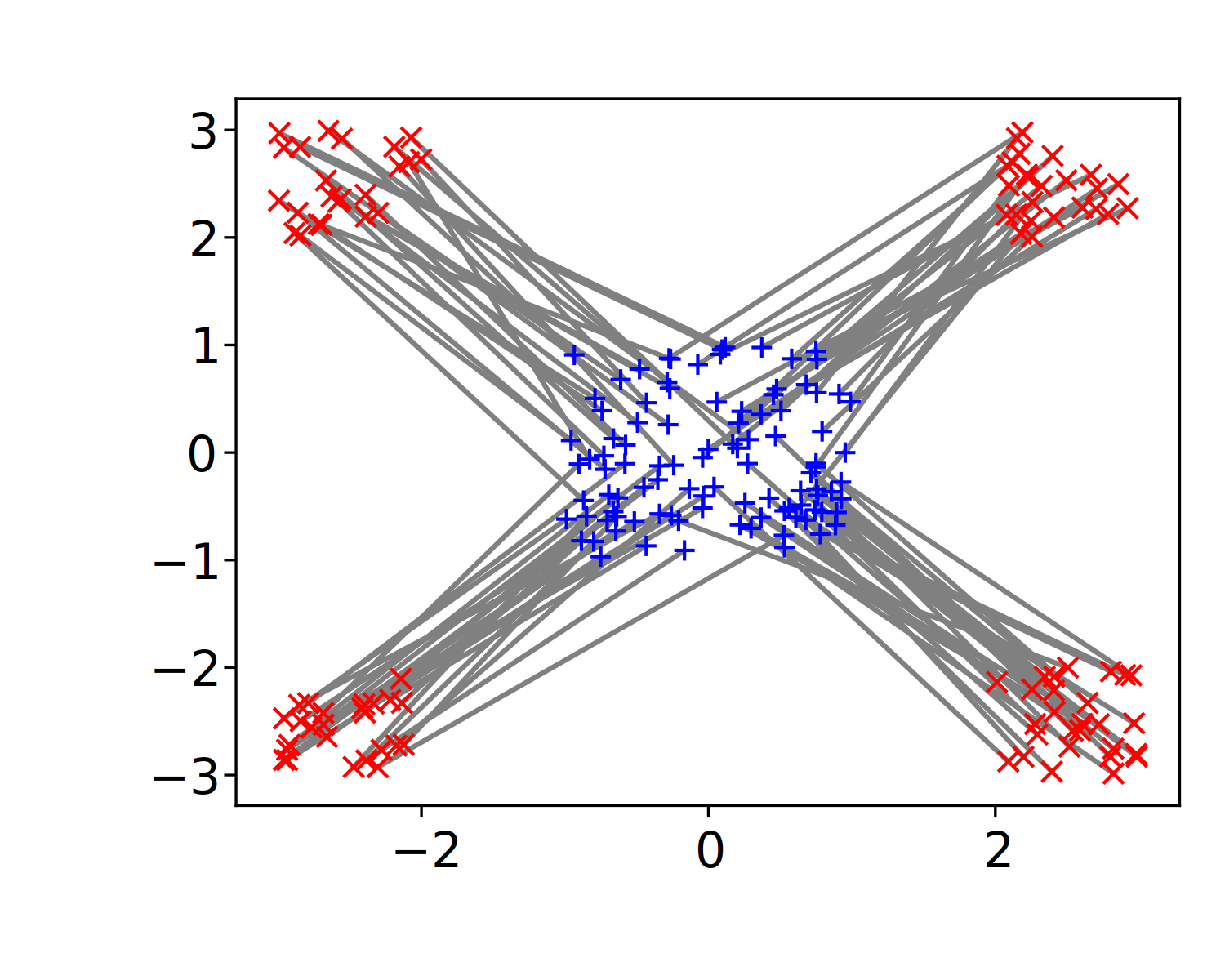}
	\end{subfigure}
	\begin{subfigure}[t]{0.3\textwidth}
		\includegraphics[width=\textwidth]{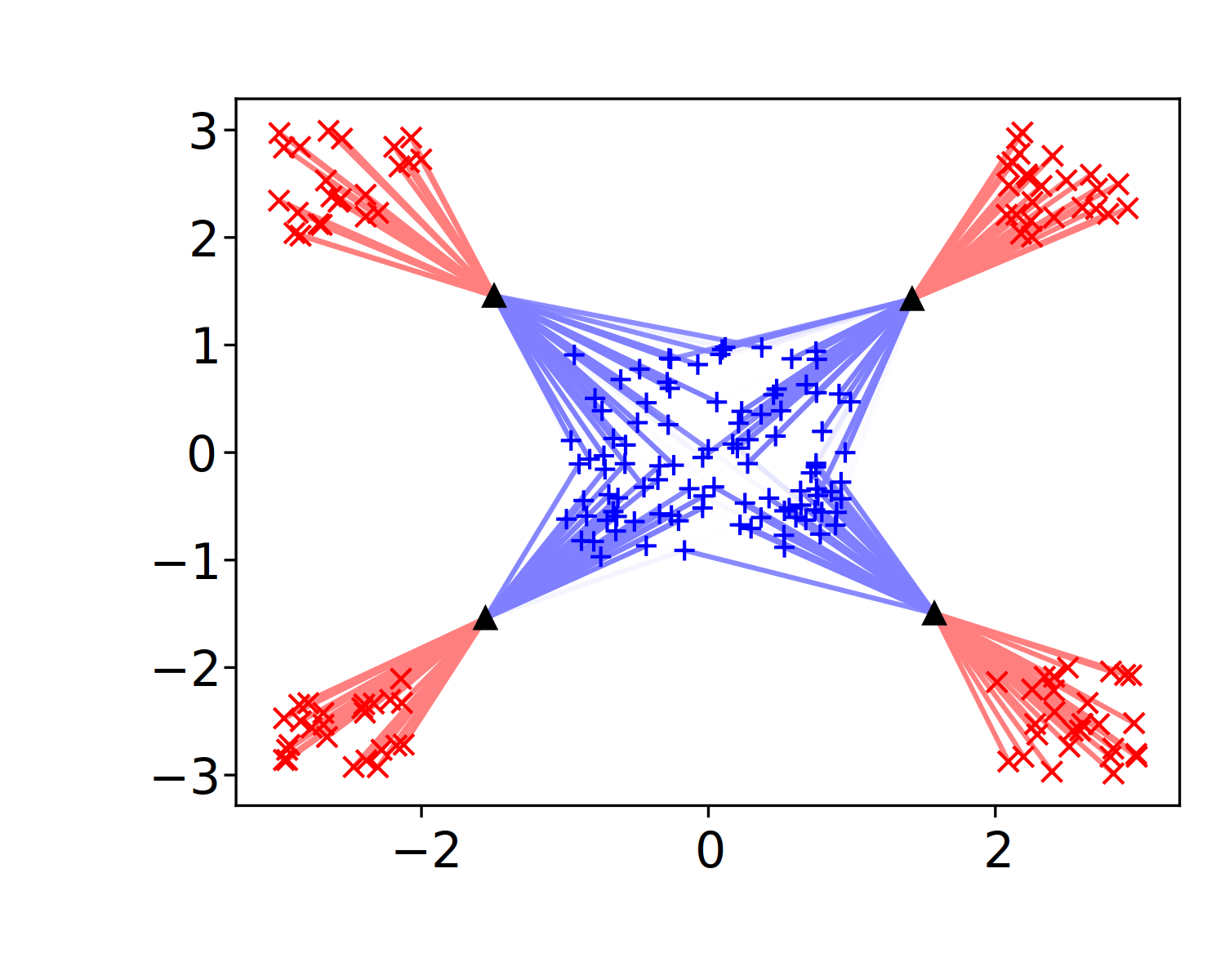}
	\end{subfigure}
	\begin{subfigure}[t]{0.3\textwidth}
		\includegraphics[width=\textwidth]{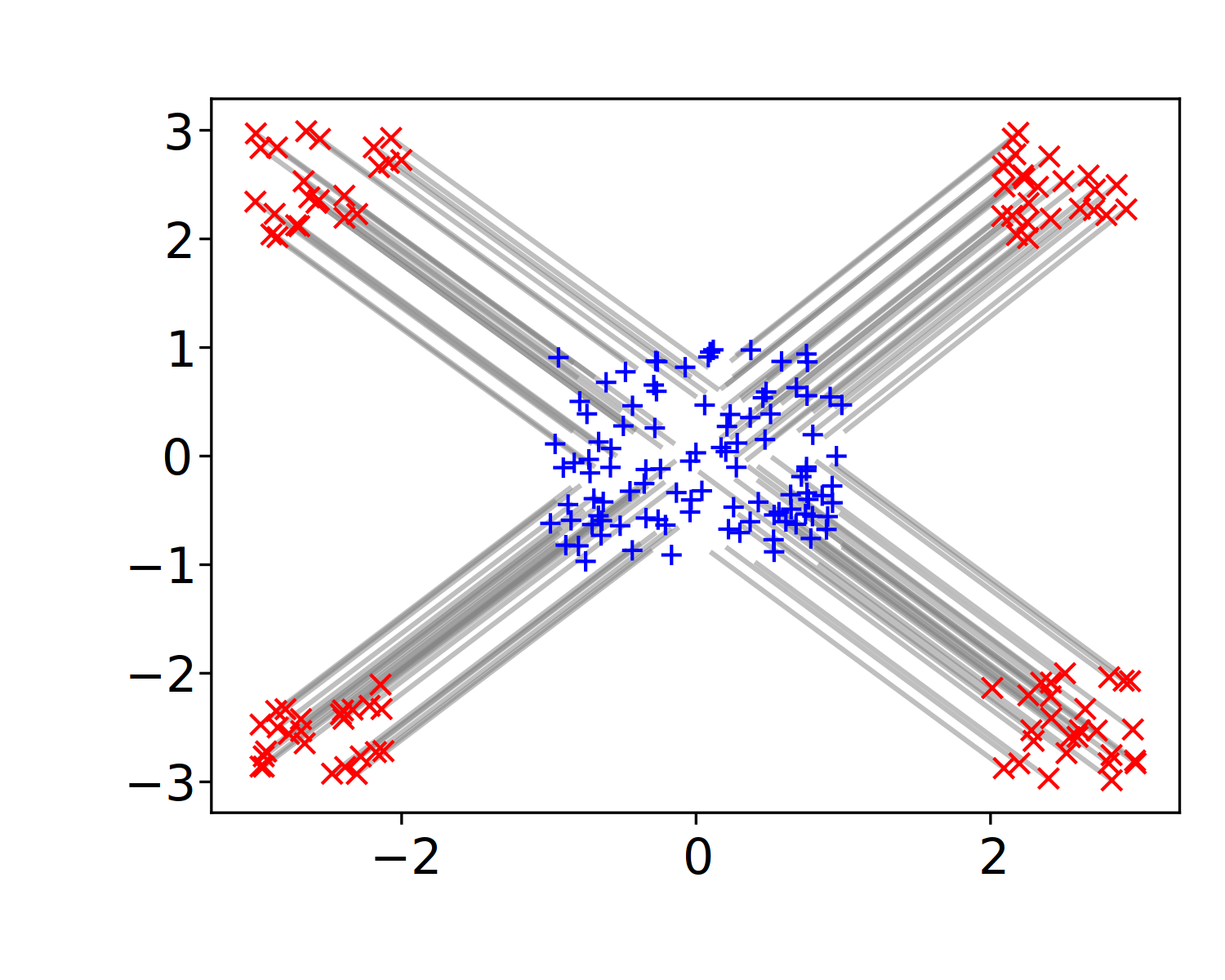}
	\end{subfigure}
	\begin{subfigure}[t]{0.3\textwidth}
		\includegraphics[width=\textwidth]{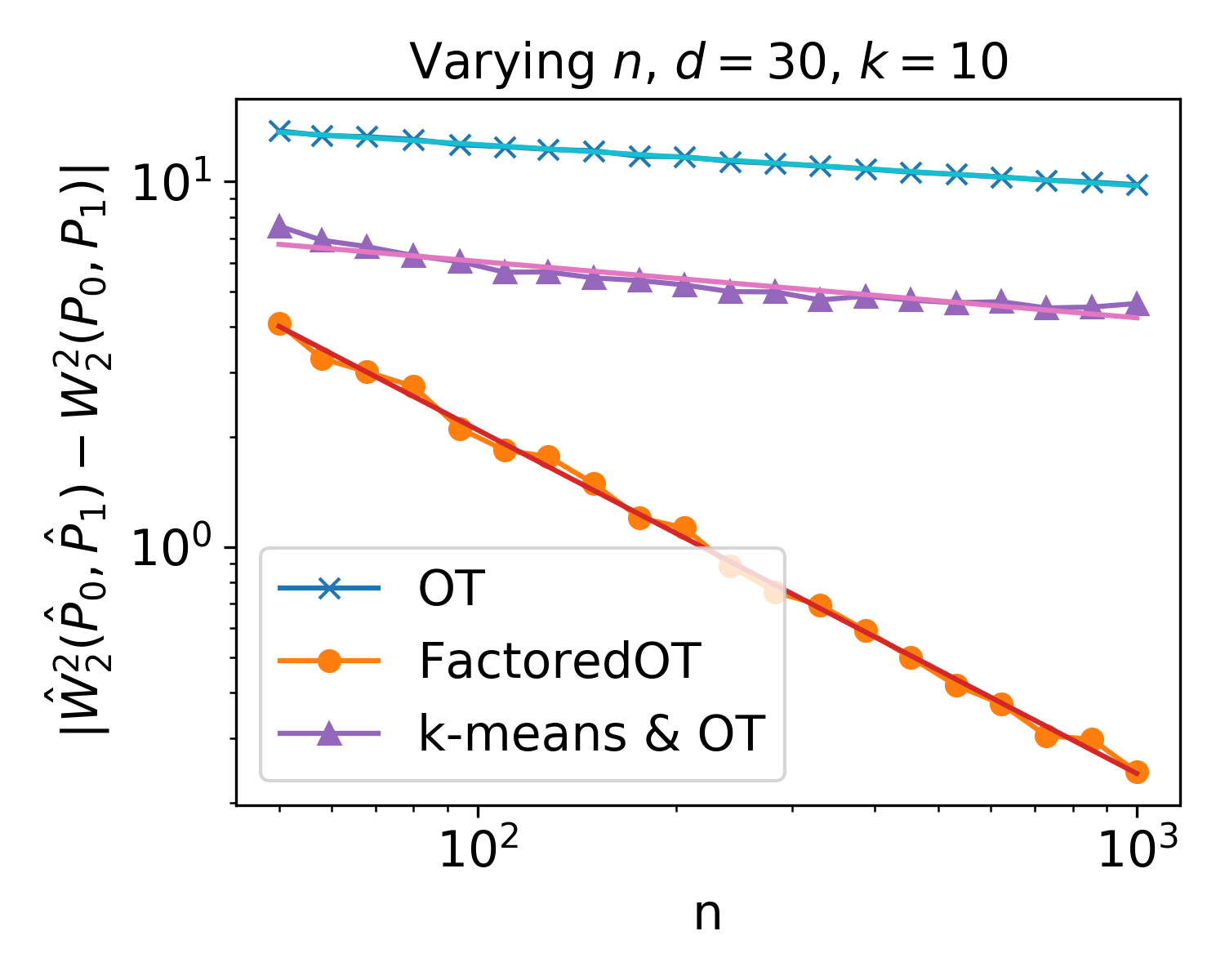}
	\end{subfigure}
	\begin{subfigure}[t]{0.3\textwidth}
		\includegraphics[width=\textwidth]{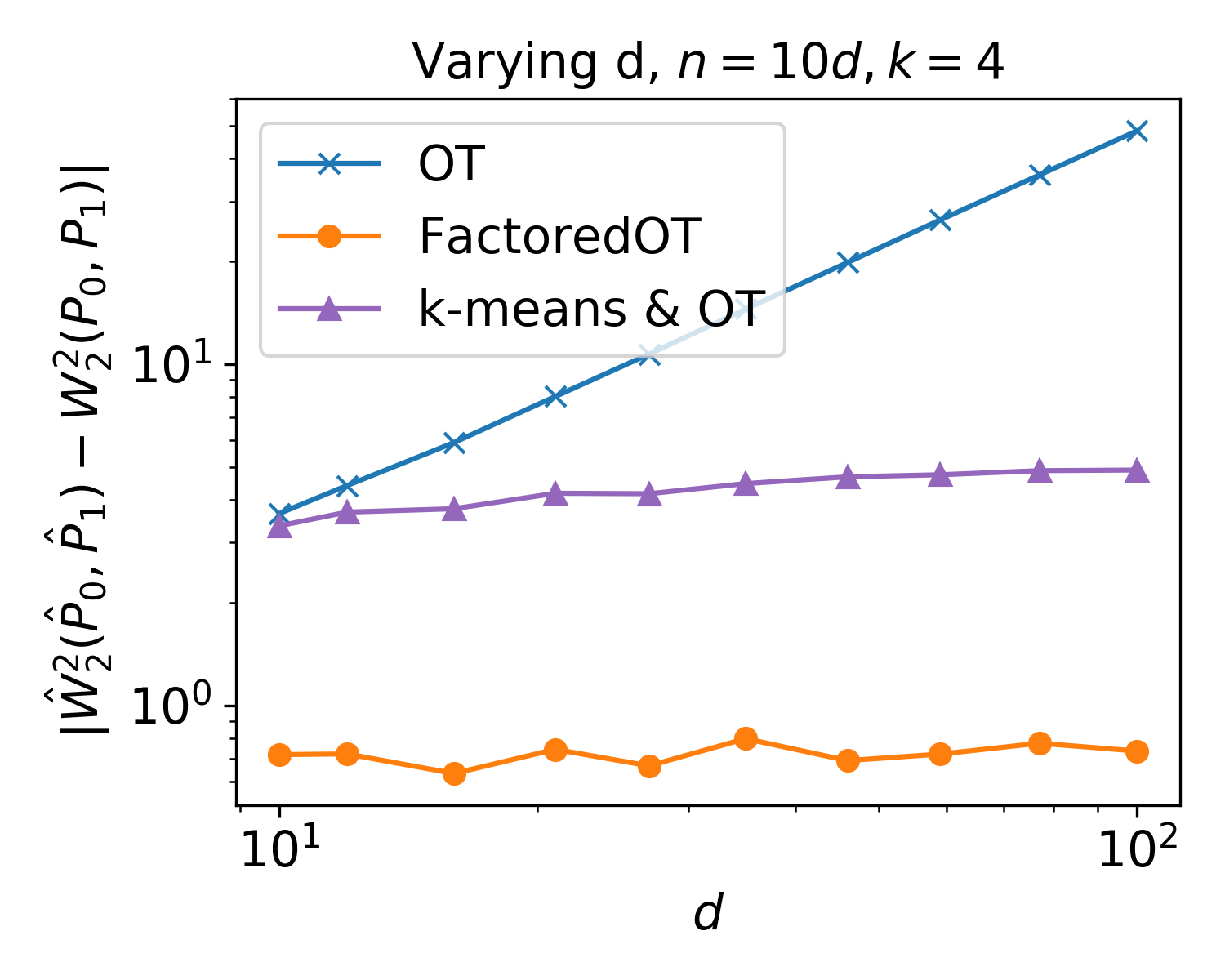}
	\end{subfigure}
	\begin{subfigure}[t]{0.3\textwidth}
		\includegraphics[width=\textwidth]{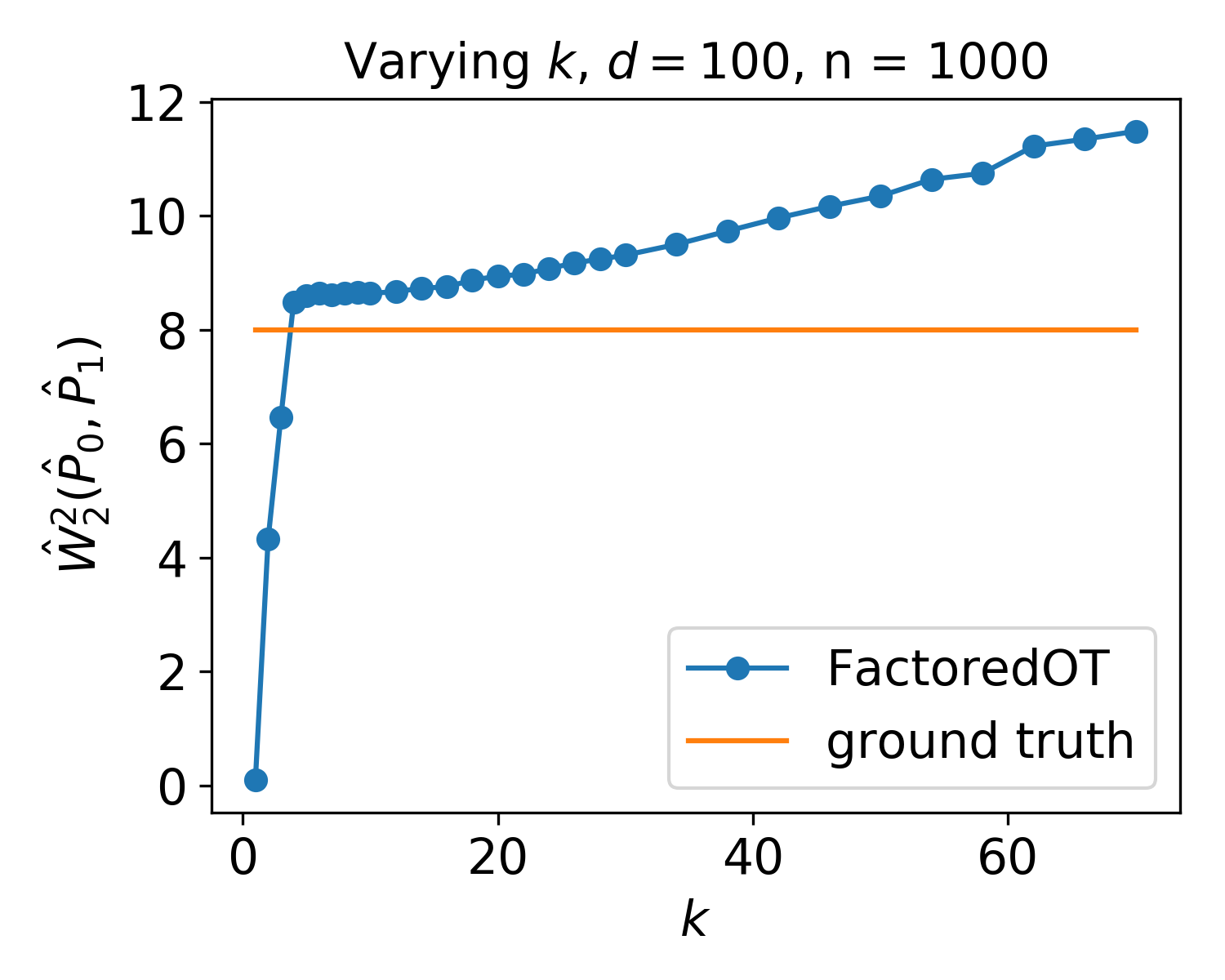}
	\end{subfigure}
	\caption{Fragmenting hypercube example.
		{\bf Top row:} Projections to the first two dimensions (computed for \( d = 30 \)) of (left) the OT coupling of samples from $P_0$ (in blue) to samples from $P_1$ (red), (middle) the FactoredOT coupling (factors in black), and (right) the FactoredOT coupling rounded to a map.  
	{\bf Bottom row:} Performance comparisons for (left) varying $n$ and (middle) varying $ d $ with $ n = 10d $, as well as (right) a diagnostic plot with varying $k$. All points are averages over 20 samples.}
    \label{FigHypercube}
    \end{center}
\end{figure*}



\begin{figure*}[h]
\begin{center}
	\begin{subfigure}[t]{0.3\textwidth}
		\includegraphics[width=\textwidth]{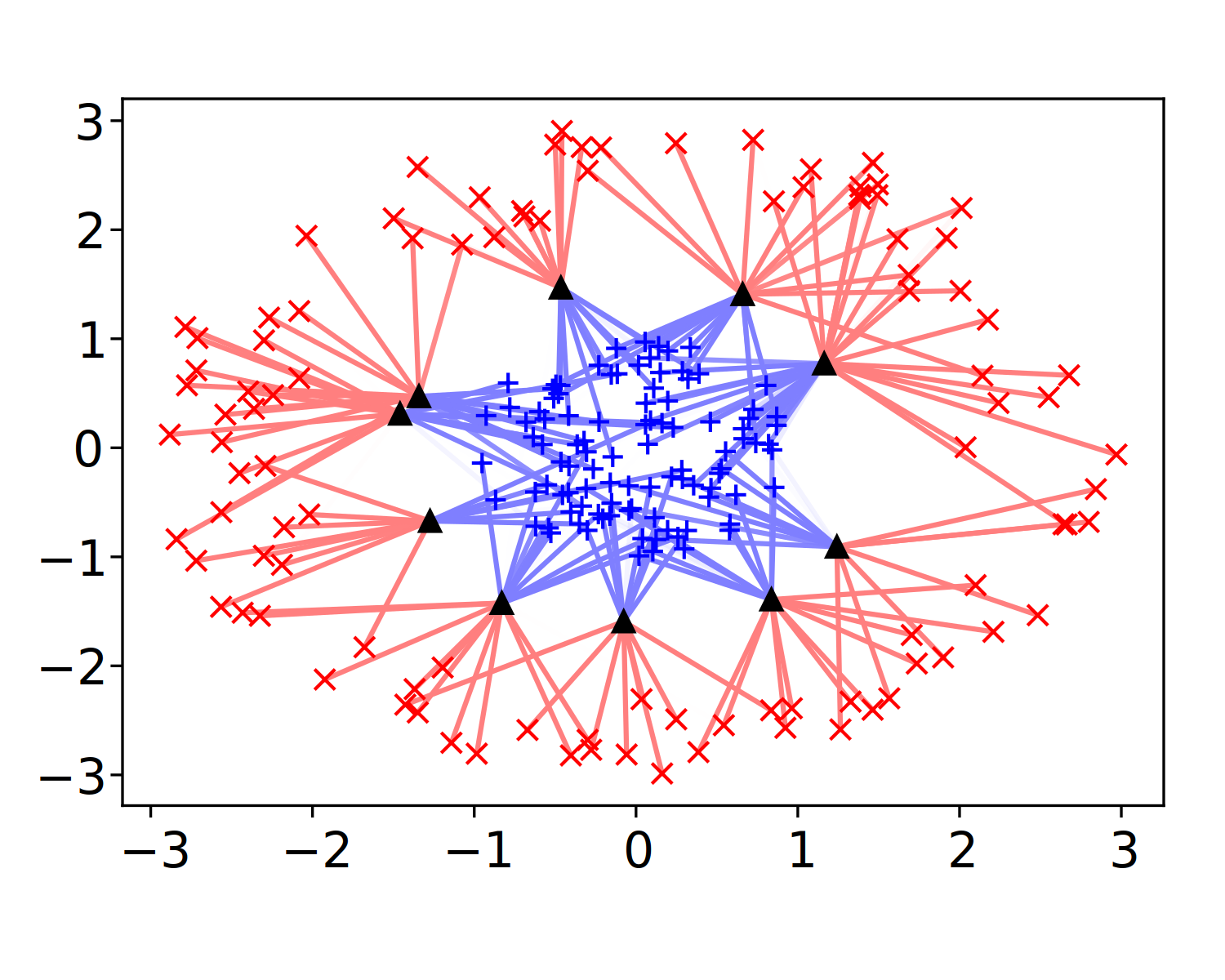}
	\end{subfigure}
	\begin{subfigure}[t]{0.3\textwidth}
		\includegraphics[width=\textwidth]{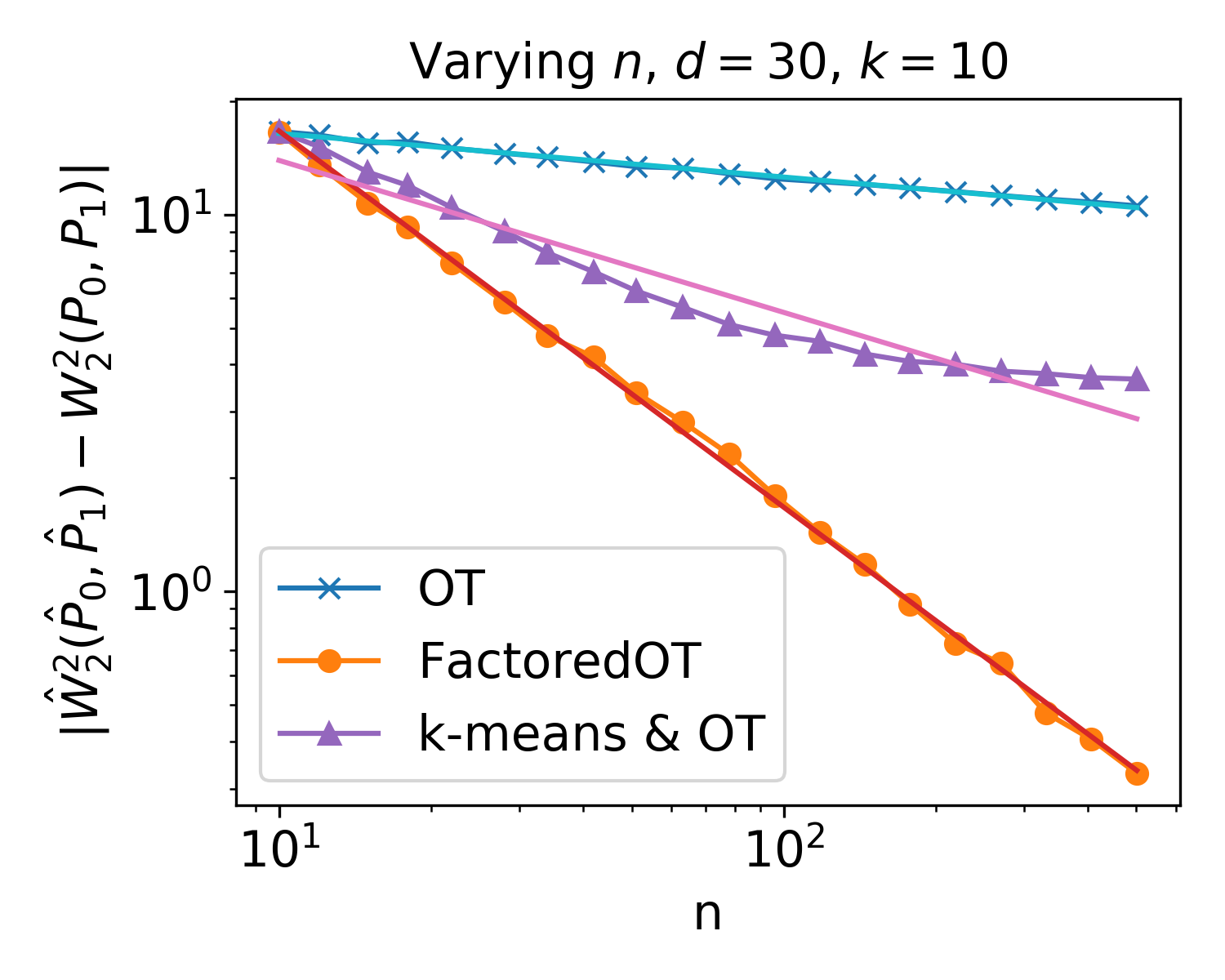}
	\end{subfigure}
	\begin{subfigure}[t]{0.3\textwidth}
		\includegraphics[width=\textwidth]{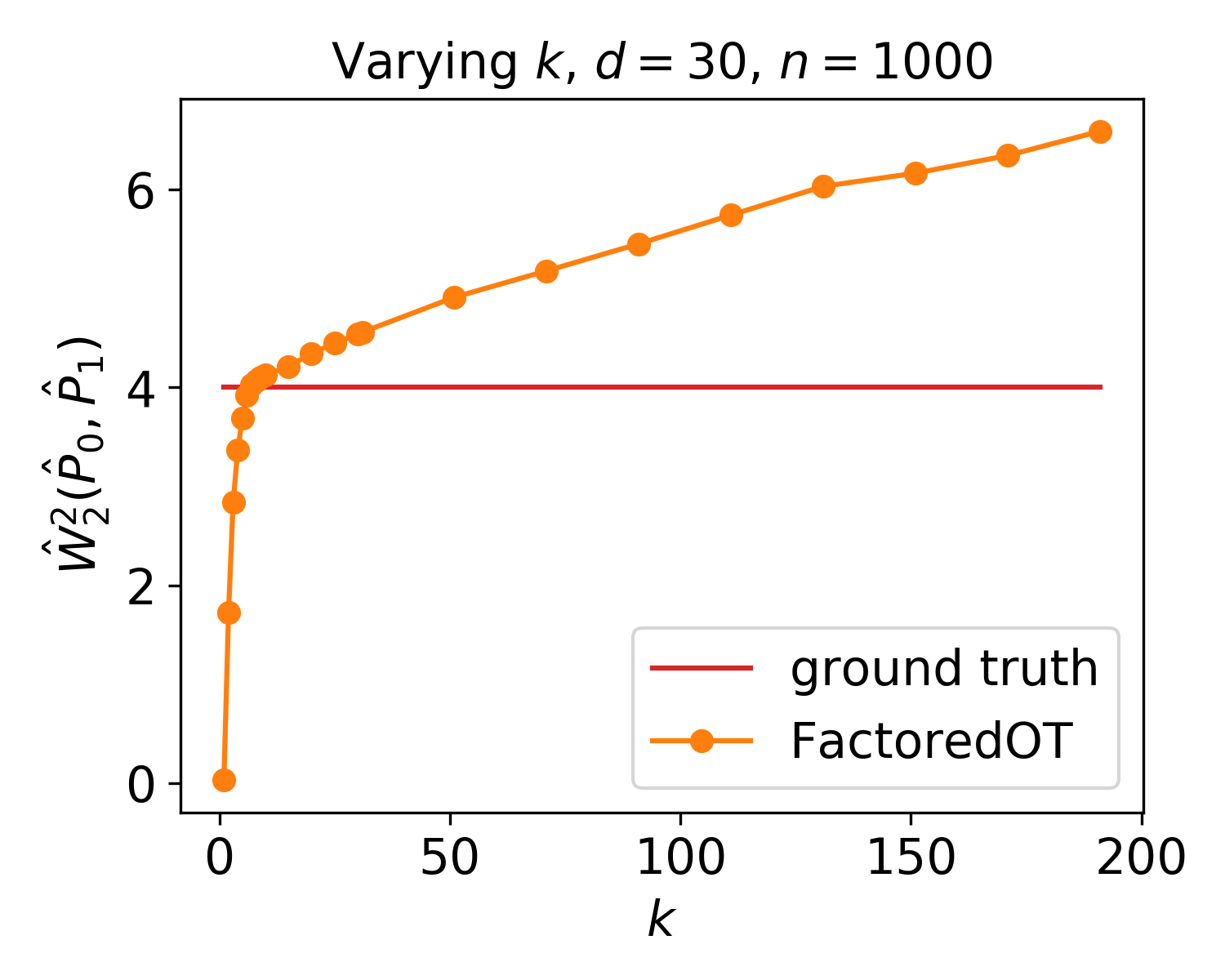}
	\end{subfigure}
	\caption{Disk to annulus example, \( d = 30 \).   
	{\bf	 Left:}  Visualization of the cluster assignment in first two dimensions. {\bf Middle:} Performance for varying \( n \).  {\bf Right:} Diagnostic plot when varying \( k \).}
    \label{FigAnnulus}
    \end{center}
\end{figure*}

\section{EXPERIMENTS}\label{sec:experiments}
 We illustrate our theoretical results with numerical experiments on both simulated and real high-dimensional data. 

For further details about the experimental setup, we refer the reader to Section \ref{sec:numerics-details} of the appendix.

\subsection{Synthetic data}
\label{sec:hypercube}

We illustrate the improved performance of our estimator for the \( W_2 \) distance on two synthetic examples.

\paragraph{Fragmented hypercube} We consider \( P_0 = \textsf{Unif}([-1,1]^d) \), the uniform distribution on a hypercube in dimension \( d \) and \( P_1 = T_\#(P_0) \), the push-forward of \( P_0 \) under a map \( T \), defined as the distribution of \( Y = T(X) \), if \( X \sim P_0 \).
We choose $T(X) = X + 2 \sign(X) \odot (e_1 + e_2)$, where the sign is taken element-wise.
As can be seen in Figure~\ref{FigHypercube}, this splits the cube into four pieces which drift away. 
This map is the subgradient of a convex function and hence an optimal transport map by Brenier's Theorem \citep[Theorem 2.12]{villani-2003}.
This observation allows us to compute explicitly $W_2^2(P_0,P_1)=8$.
We compare the results of computing optimal transport on samples and the associated empirical optimal transport cost with the estimator \eqref{eq:estimator}, as well as with a simplified procedure that consists in first performing $k$-means on both \( \hat P_0 \) and \( \hat P_1 \) and subsequently calculating the \( W_2 \) distance between the centroids.

The bottom left subplot of Figure~\ref{FigHypercube} shows that FactoredOT provides a substantially better estimate of the $W_2$ distance compared to the empirical optimal transport cost, especially in terms of its scaling with the sample size.
Moreover, from the bottom center subplot of the same figure, we deduce that a linear scaling of samples with respect to the dimension is enough to guarantee bounded error, while in the case of an empirical coupling, we see a growing error. Finally, the bottom right plot indicates that the estimator is rather stable to the choice of $k$ above a minimum threshold. We suggest choosing $k$ to match this threshold.

\paragraph{Disk to annulus}
To show the robustness of our estimator in the case where the ground truth Wasserstein distance is not exactly the cost of a factored coupling, we calculate the optimal transport between the uniform measures on a disk and on an annulus.
In order to turn this into a high-dimensional problem, we consider the 2D disk and annulus as embedded in \( d \) dimensions and extend both source and target distribution to be independent and uniformly distributed on the remaining \( d - 2 \) dimensions.
In other words, we set
\begin{gather*}
	\begin{aligned}
	P_0 = \textsf{Unif}(\{ x \in \mathbb{R}^d : {} & \| (x_1, x_2) \|_2 \le 1,\\
																								 & x_i \in [0,1] \text{ for } i = 3, \dots, d \})\\
	\end{aligned}\\
	\begin{aligned}
	P_1 = \textsf{Unif}(\{ x \in \mathbb{R}^d : {} & 2 \le \| (x_1, x_2) \|_2 \le 3,\\
																								 & x_i \in [0,1] \text{ for } i = 3, \dots, d \})
	\end{aligned}
\end{gather*}
Figure \ref{FigAnnulus} shows that the performance
is similar to that obtained for the fragmenting hypercube.

\subsection{Batch correction for single cell RNA data}
\label{sec:rna}

The advent of single cell RNA sequencing is revolutionizing biology with a data deluge. 
Biologists can now quantify the cell types that make up different tissues and quantify the molecular changes that govern development (reviewed in \cite{wagner2016revealing} and \cite{kolodziejczyk2015technology}).
As data is collected by different labs, and for different organisms, there is an urgent need for methods to robustly integrate and align these different datasets~\citep{butler2018integrating,haghverdi2018batch,crow2018characterizing}.

Cells are represented mathematically as points in a several-thousand dimensional vector space, with a dimension for each gene. The value of each coordinate represents the expression-level of the corresponding gene. 
Here we show that optimal transport achieves state of the art results for the task of aligning single cell datasets.
We align a pair of haematopoietic datasets collected by different scRNA-seq protocols in different laboratories (as described in~\cite{haghverdi2018batch}). We quantify performance by measuring the fidelity of cell-type label transfer across data sets. This information is available as ground truth in both datasets, but is not involved in computing the alignment.  

\begin{figure}[t]
	\begin{center}
		\includegraphics[width=0.7\linewidth]{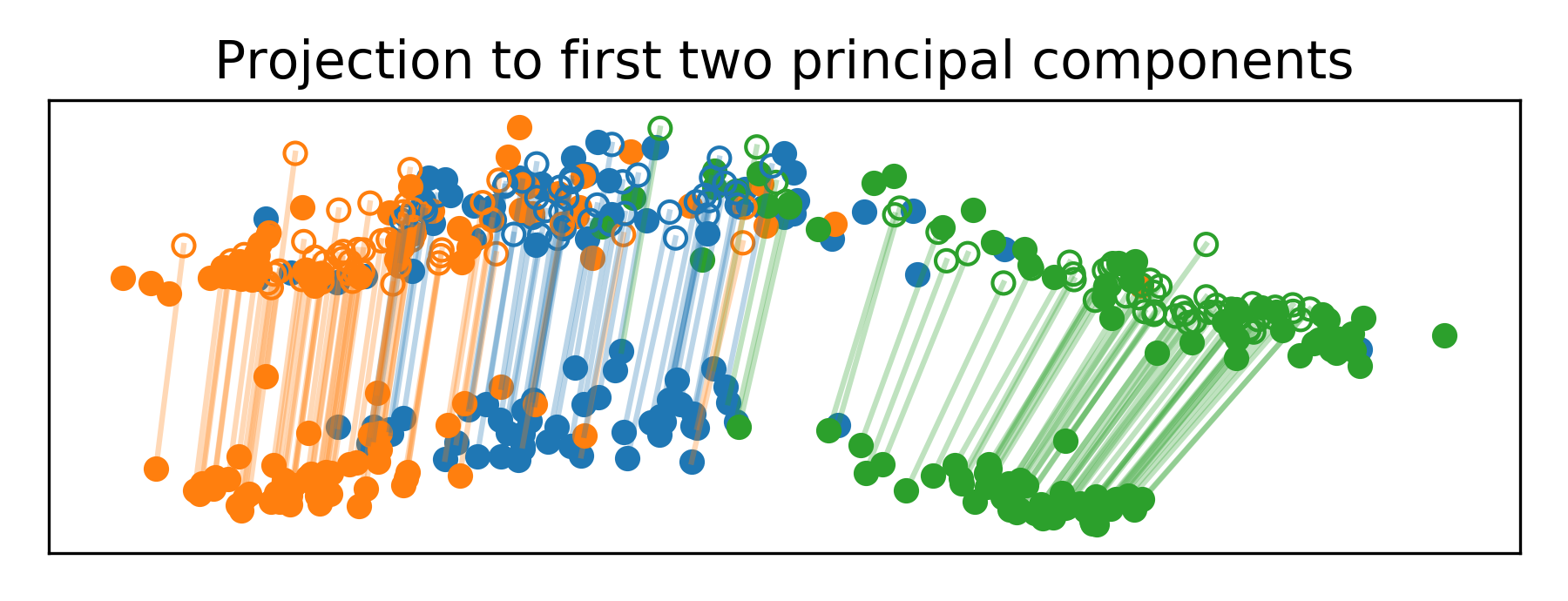}
	\end{center}
	\caption{Domain adaptation for scRNA-seq. Both source and target data set are subsampled (50 cells/type) and colored by cell type. Empty circles indicate the inferred label with 20NN classification after FactoredOT.}
	\label{fig:domain-adaptation}
\end{figure}

\newcommand\mc[1]{\multicolumn{1}{c}{#1}}
\newcommand\leftalign[1]{\textbf{#1}}
\begin{wrapfigure}{R}{0.5\textwidth}
	\begin{minipage}{0.5\textwidth}
\begin{table}[H]
	\caption{Mean mis-classification percentage (Error) and standard deviation (Std) for scRNA-Seq batch correction}
\begin{center}
\begin{tabular}{l S S}
	\toprule
	Method & \mc{Error} & \mc{Std}\\
	\midrule
	\leftalign{FOT} & \textbf{14.10} & 4.44\\
	\leftalign{MNN} & 17.53 & 5.09\\
	\leftalign{OT} & 17.47 & 3.17\\
	\leftalign{OT-ER} & 18.58 & 6.57\\
	\leftalign{OT-L1L2} & 15.47 & 5.35\\
	\leftalign{kOT} & 15.37 & 4.76\\
	\leftalign{SA} & 15.10 & 3.14\\
	\leftalign{TCA} & 24.57 & 7.04\\
	\leftalign{NN} & 21.98 & 4.90\\
	\bottomrule
\end{tabular}
\end{center}
\label{scTable}
\end{table}
\end{minipage}
\end{wrapfigure}

We compare the performance of FactoredOT (FOT) to the following baselines: (a) independent majority vote on k nearest neighbors in the target set (NN), (b) optimal transport (OT), (c) entropically regularized optimal transport (OT-ER), (d) OT with group lasso penalty (OT-L1L2) \citep{CouFlaTui14}, (e) two-step method in which we first perform $k$-means and then perform optimal transport on the $k$-means centroids (kOT), (f) Subspace Alignment (SA) \citep{FerHabSeb13}, (g) Transfer Component Analysis (TCA) \citep{PanTsaKwo11}, and (h) mutual nearest neighbors (MNN) \citep{haghverdi2018batch}.
After projecting the source data onto the target set space, we predict the label of each of the source single cells by using a majority vote over the 20 nearest neighbor single cells in the target dataset (see Figure \ref{fig:domain-adaptation} for an example).
FactoredOT outperforms the baselines for this task, as shown in Table 1, where we report the percentage of mislabeled data.

\section{DISCUSSION}\label{sec:discussion}

In this paper, we make a first step towards statistical regularization of optimal transport
with the objective of both estimating the Wasserstein distance and the optimal coupling between two probability distributions. 
Our proposed methodology generically applies to various tasks associated to optimal transport, leads to a good estimator of the Wasserstein distance even in high dimension, and is also competitive with state-of-the-art domain adaptation techniques. Our theoretical results demonstrate that the curse of dimensionality in statistical optimal transport can be overcome by imposing structural assumptions.
This is an encouraging step towards the deployment of optimal transport as a tool for high-dimensional data analysis.

Statistical regularization of optimal transport remains largely unexplored and many other forms of inductive bias may be envisioned. For example, $k$-means OT used in Section~\ref{sec:experiments} implicitly assumes that marginals are clustered--e.g., coming from a mixture of Gaussians. In this work we opt for a regularization of the optimal coupling itself, which could be accomplished in other ways. Indeed, while Theorem~\ref{thm:empirical_process} indicates that factored couplings overcome the curse of dimensionality, latent distributions with infinite support but low complexity are likely to lead to similar improvements.

\newpage

\newpage

\bibliographystyle{abbrvnat}
\bibliography{OTNotes}

\newpage

\appendix

\newcommand{\reffacbias}{\ref{prop:factored_bias_variance}}
\section{Proof of Proposition~\protect\reffacbias}
By the identification in Proposition~\ref{prop:factored_to_cluster}, we have $\gamma = \sum_{j=1}^k \frac{1}{\lambda_j} C_j^0 \otimes C_j^1$.
We perform a bias-variance decomposition:
\begin{align*}
	\leadeq{\int \|x - y\|^2 \, \mathrm{d}\gamma(x, y)}\\
	= {} & \sum_{j=1}^k \frac{1}{\lambda_j} \int \|x - y\|^2 \, \mathrm{d}C_j^0(x) \mathrm{d}C_j^1(y) \\
	= {} &\sum_{j=1}^k \frac{1}{\lambda_j} \int \|x - \mu(C_j^0) - (y - \mu(C_j^1)) + (\mu(C_j^0) - \mu(C_j^1))\|^2 \, \mathrm{d}C_j^0(x) \mathrm{d}C_j^1(y) \\
	= {} & \sum_{j=1}^k \int \|x - \mu(C_j^0)\|^2 \, \mathrm{d}C_j^0(x)  + \int  \|y - \mu(C_j^1)\|^2 \mathrm{d}C_j^1(y) + \lambda_j \|\mu(C_j^0) - \mu(C_j^1)\|^2\,,
\end{align*}
where the cross terms vanish by the definition of $\mu(C_j^0)$ and $\mu(C_j^1)$. \qed

\newcommand{\refhubsfac}{\ref{prop:hubs_to_factored}}
\section{Proof of Proposition~\protect\refhubsfac}
We first show that if $H$ is an optimal solution to~\eqref{eq:hubs}, then the hubs $z_1, \dots, z_k$ satisfy $z_j = \frac 12 (\mu(C_j^0) + \mu(C_j^1))$ for $j = 1, \dots k$.
Let $P$ be any distribution in $\cD_k$. Denote the support of $P$ by $z_1, \dots, z_k$, and let $\{C_j^0\}, \{C_j^1\}$ be the partition of $\hat P_0$ and $\hat P_1$ induced by the objective $W_2^2(P, \hat P_0) + W_2^2(P, \hat P_1)$.
By the same bias-variance decomposition as in the proof of Proposition~\ref{prop:factored_bias_variance},
\begin{equation*}
W_2^2(\hat P_0, P) = \sum_{j=1}^k \int_{C_j^0} \|x - z_j\|^2 \,\mathrm{d}\hat P_0(x) = \sum_{j=1}^k \int_{C_j^0} \|x - \mu(C_j^0)\|^2 \,\mathrm{d}\hat P_0(x) + \lambda_j \|z_j - \mu(C_j^0)\|^2\,,
\end{equation*}
and since the analogous claim holds for $\hat P_1$, we obtain that
\begin{align*}
W_2^2(P, \hat P_0) + W_2^2(P, \hat P_1)
= {} & \sum_{j=1}^k \int_{C_j^0} \|x - \mu(C_j^0)\|^2 \,\mathrm{d}\hat P_0(x)\\
{} & + \int_{C_j^1} \|y - \mu(C_j^1)\|^2 \,\mathrm{d}\hat P_1(y) + \lambda_j (\|z_j - \mu(C_j^0)\|^2 + \|z_j - \mu(C_j^1)\|^2)\,.
\end{align*}
The first two terms depend only on the partitions of $\hat P_0$ and $\hat P_1$, and examining the final term shows that any minimizer of $W_2^2(P, \hat P_0) + W_2^2(P, \hat P_1)$ must have $z_j = \frac 12 (\mu(C_j^0) + \mu(C_j^1))$ for $j = 1, \dots k$, where $C_j^0$ and $C_j^1$ are induced by $P$, in which case $\|z_j - \mu(C_j^0)\|^2 + \|z_j - \mu(C_j^1)\|^2 = \frac 12 \|\mu(C_j^0) - \mu(C_j^1)\|^2$.
Minimizing over $P \in \cD_k$ yields the claim. \qed

\newcommand{\refempproc}{\ref{thm:empirical_process}}
\section{Proof of Theorem~\protect\refempproc}
The proof of Theorem~\ref{thm:empirical_process} relies on the following propositions, which shows that controlling the gap between $W_2^2(\rho, P)$ and $W_2^2(\rho, Q)$ is equivalent to controlling the distance between $P$ and $Q$ with respect to a simple integral probability metric~\citep{Mul97}.

We make the following definition.
\begin{definition}\label{def:n_poly}
A set $S \in \R^d$ is a \emph{$n$-polyhedron} if $S$ can be written as the intersection of $n$ closed half-spaces.
\end{definition}
We denote the set of $n$-polyhedra by $\cP_n$.
Given $c \in \R^d$ and $S \in \cP_{k-1}$, define
\begin{equation*}
f_{c, S}(x) := \|x - c\|^2\indic{x \in S} \quad \quad \forall x \in \R^d\,.
\end{equation*}

\begin{prop}\label{prop:measures_to_partitions}
Let $P$ and $Q$ be probability measures supported on the unit ball in $\R^d$.
The
\begin{equation}
\sup_{\rho \in \cD_k} |W^2_2(\rho, P) - W^2_2(\rho, Q)| \leq 5 k \sup_{c: \|c\| \leq 1, S \in \cP_{k-1}} |\E_{P} f_{c, S} - \E_{Q} f_{c, S}|\,. \label{eq:ipm}
\end{equation}
\end{prop}

To obtain Theorem~\ref{thm:empirical_process}, we use techniques from empirical process theory to control the right side of~\eqref{eq:ipm} when $Q = \hat P$.
\begin{prop}\label{prop:complexity_bound}
There exists a universal constant $C$ such that, if $P$ is supported on the unit ball and $X_1, \dots, X_n \sim \mu$ are i.i.d., then
\begin{equation}
\E \sup_{c: \|c\| \leq 1, S \in \cP_{k-1}} |\E_{P} f_{c, S} - \E_{\hat P} f_{c, S}| \leq C \sqrt{\frac{kd \log k}{n}}\,.\label{eq:process_bound} 
\end{equation}
\end{prop}

With these tools in hand, the proof of Theorem~\ref{thm:empirical_process} is elementary.
\begin{proof}[Proof of Theorem~\ref{thm:empirical_process}]
Proposition~\ref{prop:measures_to_partitions} implies that
\begin{equation*}
\E \sup_{\rho \in \cD_k} |W_2^2(\rho, \hat \mu) - W_2^2(\rho, \mu)| \lesssim \sqrt{\frac{k^3 d \log k}{n}}\,.
\end{equation*}
To show the high probability bound, it suffices to apply the bounded difference inequality (see~\citep{McD89}) and note that, if $\hat P$ and $\tilde P$ differ in the location of a single sample, then for any $\rho$, we have the bound $|W_2^2(\rho, \hat P) - W_2^2(\rho, \tilde P)| \leq 4/n$. The concentration inequality immediately follows.
\end{proof}

We now turn to the proofs of Propositions~\ref{prop:measures_to_partitions} and Propositions~\ref{prop:complexity_bound}.

We first review some facts from the literature.
It is by now well known that there is an intimate connection between the $k$-means objective and the squared Wasserstein $2$-distance~\citep{Pol82,Ng00,CanRos12}. This correspondence is based on the following observation, more details about which can be found in~\citep{GraLus00}: given fixed points $c_1, \dots, c_k$ and a measure $P$, consider the quantity
\begin{equation}
\min_{w \in \Delta_k} W_2^2\Big(\sum_{i=1}^k w_i \delta_{c_i}, P\Big)\,, \label{eq:free_weights}
\end{equation}
where the minimization is taken over all probability vectors $w := (w_1, \dots, w_k)$. Note that, for \emph{any} measure $\rho$ supported on $\{c_1, \dots, c_k\}$, we have the bound
\begin{equation*}
W_2^2(\rho, P) \geq \E \big[\min_{i \in [k]} \|X - c_k\|^2\big] \quad \quad X \sim P\,.
\end{equation*}
On the other hand, this minimum can be achieved by the following construction. Denote by $\{S_1, \dots, S_k\}$ the Voronoi partition~\citep{OkaBooSug00} of $\R^d$ with respect to the centers $\{c_1, \dots, c_k\}$ and let $\rho = \sum_{i=1}^k P(S_i) \delta_{c_i}$. 
If we let $T: \R^d \to \{c_1, \dots, c_k\}$ be the function defined by $S_i = T^{-1}(c_i)$ for $i \in [k]$, then $(\mathrm{id}, T)_\sharp P$ defines a coupling between $P$ and $\rho$ which achieves the above minimum, and
\begin{equation*}
\E[\|X - T(X)\|^2] = \E [\min_{i \in [k]} \|X - c_i\|^2] \quad \quad X \sim P\,.
\end{equation*}

The above argument establishes that the measure closest to $P$ with prescribed support of at most $k$ points is induced by a Voronoi partition of $\R^d$, and this observation carries over into the context of the $k$-means problem~\citep{CanRos12}, where one seeks to solve
\begin{equation}
\min_{\rho \in \cD_k} W_2^2(\rho, P)\,. \label{eq:free_weights_and_centers}
\end{equation}
The above considerations imply that the minimizing measure will correspond to a Voronoi partition, and that the centers $c_1, \dots, c_k$ will lie at the centroids of each set in the partition with respect to $P$. As above, there will exist a map $T$ realizing the optimal coupling between $P$ and $\rho$, where the sets $T^{-1}(c_i)$ for $i \in [k]$ form a Voronoi partition of $\R^d$. In particular, standard facts about Voronoi cells for the $\ell_2$ distance~\citep[Definition~V4]{OkaBooSug00} imply that, for $i \in [k]$, the set $\cl(T^{-1}(c_i))$ is a $(k-1)$-polyhedron. (See Definition~\ref{def:n_poly} above.)

In the case when $\rho$ is an \emph{arbitrary} measure with support of size $k$---and not the solution to an optimization problem such as~\eqref{eq:free_weights} or~\eqref{eq:free_weights_and_centers}---it is no longer the case that the optimal coupling between $P$ and $\rho$ corresponds to a Voronoi partition of $\R^d$. The remainder of this section establishes, however, that, if $P$ is absolutely continuous with respect to the Lebgesgue measure, then there does exist a map $T$ such that the fibers of points in the image of $T$ have a particularly simple form: like Voronoi cells, the sets $\{\cl(T^{-1}(c_i)\}_{i=1}^k$ can be taken to be simple polyehdra.

\begin{definition}
A function $T: \R^d \to \R^d$ is a \emph{polyhedral quantizer} of order $k$ if $T$ takes at most $k$ values and if, for each $x \in \im(T)$, the set $\cl(T^{-1}(x))$ is a $(k-1)$-polyhedron and $\partial T^{-1}(x)$ has zero Lebesgue measure.
\end{definition}
We denote by $\cQ_k$ the set of $k$-polyhedral quantizers whose image lies inside the unit ball of $\R^d$.

\begin{prop}\label{prop:convex_partition}
Let $P$ be any absolutely continuous measure in $\R^d$, and let $\rho$ be any measure supported on $k$ points.
Then there exists a map $T$ such that $(id, T)_\sharp P$ is an optimal coupling between $P$ and $\rho$ and $T$ is a polyhedral quantizer of order $k$.
\end{prop}
\begin{proof}
Denote by $\rho_1, \dots, \rho_k$ the support of $\rho$.
Standard results in optimal transport theory~\citep[Theorem~1.22]{San15} imply that there exists a convex function $u$ such that the optimal coupling between $P$ and $\rho$ is of the form $(\mathrm{id}, \nabla u)_\sharp P$.
Let $S_i = (\nabla u)^{-1}(\rho_i)$.

Since $\nabla u(x) = \rho_j$ for any $x \in S_j$, the restriction of $u$ to $S_j$ must be an affine function.
We obtain that there exists a constant $\beta_j$ such that
\begin{equation*}
u(x) = \langle \rho_j, x \rangle + \beta_j \quad \quad \forall x \in S_j\,.
\end{equation*}
Since $\rho_j$ has nonzero mass, the fact that $\nabla u_\sharp P = \rho$ implies that $P(S_j) > 0$, and,
since $P$ is absolutely continuous with respect to the Lebesgue measure, this implies that $S_j$ has nonempty interior.
If $x \in \inter(S_j)$, then $\partial u(x) = \{\rho_j\}$. Equivalently, for all $y \in \R^d$,
\begin{equation*}
u(y) \geq \langle \rho_j, y \rangle + \beta_j\,.
\end{equation*}

Employing the same argument for all $j \in [k]$ yields
\begin{equation*}
u(x) \geq \max_{j \in [k]} \langle \rho_j, x \rangle + \beta_j\,.
\end{equation*}
On the other hand, if $x \in S_i$, then
\begin{equation*}
u(x) = \langle \rho_i, x \rangle + \beta_i \leq \max_{j \in [k]} \langle \rho_j, x \rangle + \beta_j\,.
\end{equation*}

We can therefore take $u$ to be the convex function
\begin{equation*}
u(x) = \max_{j \in [k]} \langle \rho_j, x \rangle + \beta_j\,,
\end{equation*}
which implies that, for $i \in [k]$,
\begin{align*}
\cl(S_i) & = \{y \in \R^d : \langle \rho_i, x \rangle + \beta_i \geq \langle \rho_j, x \rangle + \beta_j \quad \forall j \in [k]\setminus\{i\}\} \\
& = \bigcap_{j \neq i} \{y \in \R^d : \langle \rho_i, x \rangle + \beta_i \geq \langle \rho_j, x \rangle + \beta_j\}\,.
\end{align*}
Therefore $\cl(S_i)$ can be written as the intersection of $k-1$ halfspaces.
Moreover, $\partial S_i \subseteq \bigcup_{j \neq i} \{y \in \R^d : \langle \rho_i, x \rangle + \beta_i = \langle \rho_j, x \rangle + \beta_j\}$, which has zero Lebesgue measure, as claimed.
\end{proof}

\subsection{Proof of Proposition~\ref{prop:measures_to_partitions}}
By symmetry, it suffices to show the one-sided bound
\begin{equation*}
\sup_{\rho \in \cD_k} W^2_2(\rho, Q) - W^2_2(\rho, P)  \leq 5 k \sup_{c: \|c\| \leq 1, S \in \cP_{k-1}} |\E_{P} f_{c, S} - \E_{Q} f_{c, S}|\,.
\end{equation*}

We first show the claim for $P$ and $Q$ which are absolutely continuous.
Fix a $\rho \in \cD_k$. Since $P$ and $Q$ are absolutely continuous, we can apply Proposition~\ref{prop:convex_partition} to obtain that there exists a $T \in \cQ_k$ such that
\begin{equation*}
W_2^2(\rho, P) = \E_P \| X - T(X)\|^2\,.
\end{equation*}
Let $\{c_1, \dots, c_k\}$ be the image of $T$, and for $i \in [k]$ let $S_i := \cl(T^{-1}(c_i))$.
Denote by $d_{TV}(\mu, \nu) := \sup_{\text{$A$ measurable}} |\mu(A) - \nu(A)|$ the total variation distance between $\mu$ and $\nu$.
Applying Lemma~\ref{lem:tv_triangle} to $\rho$ and $Q$ yields that
\begin{equation*}
W_2^2(Q, \rho) \leq \E_Q\|X - T(X)\|^2 + 4 \mathrm{d_{TV}}(T_\sharp Q, \rho)\,.
\end{equation*}
Since $\rho = T_\sharp P$ and $Q$ and $P$ are absolutely continuous with respect to the Lebesgue measure, we have
\begin{align*}
\mathrm{d_{TV}}(T_\sharp Q, \rho) & = \mathrm{d_{TV}}(T_\sharp Q, T_\sharp P) \\
& = \frac 12 \sum_{i=1}^k |P(T^{-1}(c_i)) - Q(T^{-1}(c_i))| \\
& = \frac 12 \sum_{i=1}^k |P(S_i) - Q(S_i)|\,.
\end{align*}
Combining the above bounds yields
\begin{align*}
W_2^2(\rho, Q) - W_2^2(\rho, P) & \leq \E_Q \|X - T(X)\|^2 - \E_P \|X - T(X)\|^2 + 2 \sum_{i=1}^k |P(S_i) - Q(S_i)| \\
& \leq \sum_{i=1}^k |\E_Q\|X - c_i\|^2\indic{X \in S_i} - \E_P \|X - c_i\|^2\indic{X \in S_i}| + 2 |P(S_i) - Q(S_i)| \\
& \leq k \sup_{c, S} \left(|\E_Q\|X - c\|^2\indic{X \in S} - \E_P \|X - c\|^2\indic{X \in S}| + 2 |P(S) - Q(S)|\right) \\ 
& = k \sup_{c, S}\left(|\E_Q f_{c, S} - \E_P f_{c, S}| + 2 |P(S) - Q(S)|\right)
\end{align*}
where the supremum is taken over $c \in \R^d$ satisfying $\|c\| \leq 1$ and $S \in \cP_{k-1}$.

If $\|v\| = 1$, then
\begin{equation*}
\indic{X \in S} = \frac 1 2 (\|X + v\|^2 + \|X - v\|^2 - 2 \|X\|^2) \indic{X \in S}\,,
\end{equation*}
which implies
\begin{align*}
|P(S) - Q(S)| & = | \E_P \indic{X \in S} - \E_Q \indic{X \in S}| \\
& \leq \frac 12 \left( | \E_P f_{v, S} - \E_Q f_{v, S}| + | \E_P f_{-v, S} - \E_Q f_{-v, S}| + 2| \E_P f_{0, S} - \E_Q f_{0, S}|\right) \\
& \leq 2 \sup_{c, S} |\E_P f_{c, S} - \E_Q f_{c, S}|\,.
\end{align*}
Combining the above bounds yields
\begin{equation*}
W_2^2(\rho, Q) - W_2^2(\rho, P) \leq 5 k \sup_{c, S}|\E_P f_{c, S} - \E_Q f_{c, S}|
\end{equation*}
Finally, since this bound holds for all $\rho \in \cD_k$, taking the supremum of the left side yields the claim for absolutely continuous $P$ and $Q$.

To prove the claim for arbitrary measures, we reduce to the absolutely continuous case.
Let $\delta \in (0, 1)$ be arbitrary, and let $\cK_\delta$ be any absolutely continuous probability measure such that, if $Z \sim \cK_\delta$ then $\|Z\| \leq \delta$ almost surely.
Let $\rho \in \cD_k$. The triangle inequality for $W_2$ implies
\begin{equation*}
|W_2(\rho, Q) - W_2(\rho, Q * \cK_\delta)| \leq W_2(Q, Q * \cK_\delta) \leq \delta\,,
\end{equation*}
where the final inequality follows from the fact that, if $X \sim Q$ and $Z \sim \cK_\delta$, then $W^2_2(Q, Q* \cK_\delta) \leq \E\|X - (X+Z)\|^2 \leq \delta^2$.
Since $\rho$ and $Q$ are both supported on the unit ball, the trivial bound $W_2(\rho, Q) \leq 2$ holds.
If $\delta \leq 1$, then $W_2(\rho, Q * \cK_\delta) \leq 3$, and we obtain
\begin{equation*}
|W_2^2(\rho, Q) - W_2^2(\rho, Q * \cK_\delta)| \leq 5 \delta\,.
\end{equation*}
The same argument implies
\begin{equation*}
|W_2^2(\rho, P) - W_2^2(\rho, P * \cK_\delta)| \leq 5 \delta\,.
\end{equation*}

Therefore
\begin{equation*}
\sup_{\rho \in \cD_K} W^2_2(\rho, Q) - W^2_2(\rho, P) \leq \sup_{\rho \in \cD_K} W^2_2(\rho, Q * \cK_\delta) - W^2_2(\rho, P * \cK_\delta) + 10 \delta\,.
\end{equation*}

Likewise, for any $x$ and $c$ in the unit ball, if $\|z\| \leq \delta$, then by the exact same argument as was used above to bound $|W_2^2(\rho, Q) - W_2^2(\rho, Q * \cK_\delta)|$, we have
\begin{equation*}
|f_{c, S}(x + z) - f_{c, S - z}(x)| \leq 5\delta\,.
\end{equation*}

Let $Z \sim \cK_\delta$ be independent of all other random variables, and denote by $\E_Z$ expectation with respect to this quantity.
Now, applying the proposition to the absolutely continuous measures $P * \cK_\delta$ and $Q * \cK_\delta$, we obtain
\begin{align*}
\sup_{\rho \in \cD_k} W^2_2(\rho, Q) - W^2_2(\rho, P) & \leq 5 k \sup_{c, S} \left| \E_Z \left[\E_P f_{c, S}(X + Z) - \E_Q f_{c, S}(X + Z)\right]\right|  
+ 10 \delta \\
& \leq \E_{Z}\Big[5 k \sup_{c, S}  |\E_P f_{c, S}(X + Z) - \E_Q f_{c, S}(X + Z)| \Big] + 10 \delta \\
& \leq \E_{Z} \Big[5 k \sup_{c, S}  |\E_P f_{c, S-Z} - \E_Q f_{c, S-Z}| \Big] + 20 \delta\,.
\end{align*}
It now suffices to note that, for any $S \in \cP_{k-1}$ and any $z \in \R^d$, the set $S - z \in \cP_{k -1}$.
In particular, this implies that
\begin{equation*}
z \mapsto \sup_{c, S}  |\E_P f_{c, S-z} - \E_Q f_{c, S-z}|
\end{equation*}
is constant, so that the expectation with respect to $Z$ can be dropped.

We have shown that, for any $\delta \in (0, 1)$, the bound
\begin{align*}
\sup_{\rho \in \cD_K} W^2_2(\rho, P) - W^2_2(\rho, Q) & \leq 5 k \sup_{c, S}  |\E_P f_{c, S} - \E_Q f_{c, S}| + 20 \delta
\end{align*}
holds. Taking the infimum over $\delta > 0$ yields the claim.\qed

\newcommand{\refcompbound}{\ref{prop:complexity_bound}}
\section{Proof of {Proposition~\protect\refcompbound}}
In this proof, the symbol $C$ will stand for a universal constant whose value may change from line to line.
For convenience, we will use the notation $\sup_{c, S}$ to denote the supremum over the feasible set $c: \|c\| \leq 1, S \in \cP_{k-1}$.

We employ the method of~\citep{MauPon10}.
By a standard symmetrization argument~\citep{GinNic16}, if $g_1, \dots, g_n$ are i.i.d.\ standard Gaussian random variables, then the quantity in question is bounded from above by
\begin{equation*}
\frac{\sqrt{2 \pi}}{n} \E \sup_{c, S} \left|\sum_{i=1}^n g_i f_{c, S}(X_i)\right| \leq \frac{\sqrt{8 \pi}}{n} \E \sup_{c, S} \sum_{i=1}^n g_i f_{c, S}(X_i) + \frac{C}{\sqrt n}\,.
\end{equation*}
Given $c$ and $c''$ in the unit ball and $S, S' \in \cP_{k-1}$, consider the increment $(f_{c, S}(x) - f_{c', S'}(x))^2$. If $x \in S \triangle S'$ and $\|x\| \leq 1$, then
\begin{align*}
(f_{c, S}(x) - f_{c', S'}(x))^2 
& \leq \max \left\{\|x - c\|^4, \|x - c'\|^4\right\} \leq 16\,.
\end{align*}
On the other hand, if $x \notin S \triangle S'$, then
\begin{align*}
(f_{c, S}(x) - f_{c', S'}(x))^2 
& \leq (\|x - c\|^2 - \|x - c'\|^2)^2\,.
\end{align*}
Therefore, for any $x$ in the unit ball,
\begin{equation*}
(f_{c, S}(x) - f_{c', S'}(x))^2 \leq 16 (\indic{x \in S} - \indic{x \in S'})^2 + (\|x - c\|^2 - \|x - c'\|^2)^2\,.
\end{equation*}

This fact implies that the Gaussian processes
\begin{alignat*}{2}
G_{c, S} & := \sum_{i=1}^n g_i f_{c, S} (X_i) \quad \quad & g_i \sim \cN(0, 1) \text{ i.i.d}\\
H_{c, S} & := \sum_{i=1}^n 4 g_i  \indic{X_i \in S} + g_i' \|X_i - c\|^2 \quad \quad & g_i, g_i' \sim \cN(0, 1) \text{ i.i.d}\,,
\end{alignat*}
satisfy
\begin{equation*}
\E (G_{c, S} - G_{c', S'})^2 \leq \E (H_{c, S} - H_{c', S'})^2 \quad \quad \forall c, c', S, S'\,.
\end{equation*}

Therefore, by the Slepian-Sudakov-Fernique inequality~\citep{Sle62,Sud71,Fer75},
\begin{align*}
\E \sup_{c, S} \sum_{i=1}^n g_i f_{c, S}(X_i) & \leq \E \sup_{c, S} \sum_{i=1}^n 4 g_i  \indic{X_i \in S} + g_i' \|X_i - c\|^2 \\
& \leq \E \sup_{S \in \cP_{k-1}} 4 \sum_{i=1}^n g_i  \indic{X_i \in S} + \E \sup_{c: \|c\| \leq 1} \sum_{i=1}^n g_i \|X_i - c\|^2\,.
\end{align*}
We control the two terms separately.
The first term can be controlled using the VC dimension of the class $\cP_{k-1}$~\citep{VapCer71} by a standard argument in empirical process theory (see, e.g.,~\citep{GinNic16}).
Indeed, using the bound~\citep[Lemma~7.13]{Dud78} combined with the chaining technique~\citep{Ver16} yields
\begin{equation*}
\E \sup_{S \in \cP_{k-1}} 4 \sum_{i=1}^n g_i \indic{X_i \in S} \leq C \sqrt{n \mathrm{VC}(\cP_{k-1})}\,.
\end{equation*}

By Lemma~\ref{lem:VC}, $VC(\cP_{k-1}) \leq C d k \log k$; hence
\begin{equation*}
\E \sup_{S \in \cP_{k-1}} 4 \sum_{i=1}^n g_i \indic{X_i \in S} \leq C \sqrt{n d k \log k}\,.
\end{equation*}

The second term can be controlled as in~\citep[Lemma~3]{MauPon10}:
\begin{align*}
\E \sup_{c: \|c\| \leq 1} \sum_{i=1}^n g_i \|X_i - c\|^2 & = \E \sup_{c: \|c\| \leq 1} \sum_{i=1}^n g_i (\|X_i\|^2 -2 \langle X_i, c\rangle + \|c\|^2) \\
& \leq 2 \E \sup_{c: \|c\| \leq 1} \sum_{i=1}^n g_i \langle X_i, c\rangle + \sup_{c: \|c\| \leq 1} \sum_{i=1}^n g_i \|c\|^2 \\
& \leq 2 \E \left\| \sum_{i=1}^n g_i X_i \right\| + \left| \sum_{i=1}^n g_i \right| \\
& \leq C \sqrt n
\end{align*}
for some absolute constant $C$.

Combining the above bounds yields
\begin{equation*}
\frac{\sqrt{8 \pi}}{n} \E \sup_{c: \|c\| \leq 1, S \in \cP_{k-1}} \sum_{i=1}^n g_i f_{c, S}(X_i) \leq C \sqrt{\frac{d k \log k}{n}}\,,
\end{equation*}
and the claim follows.\qed

\section{Additional lemmas}
\begin{lemmapp}\label{lem:tv_triangle}
Let $\mu$ and $\nu$ be probability measures on $\R^d$ supported on the unit ball. If $T: \R^d \to \R^d$, then
\begin{equation*}
W_2^2(\mu, \nu) \leq \E\|X - T(X)\|^2 + 4 \mathrm{d_{TV}}(T_\sharp \mu, \nu) \quad \quad X \sim \mu\,.
\end{equation*}
\end{lemmapp}
\begin{proof}
If $X \sim \mu$, then $(X, T(X))$ is a coupling between $\mu$ and $T_\sharp \mu$. Combining this coupling with the optimal coupling between $T \sharp \mu$ and $\nu$ and applying the gluing lemma~\citep{Vil09} yields that there exists a triple $(X, T(X), Y)$ such that $X \sim \mu$, $Y \sim \nu$, and $\p[T(X) \neq Y] = \mathrm{d_{TV}}(T_\sharp \mu, \nu)$.
\begin{align*}
W_2^2(\mu, \nu) & \leq \E[\|X - Y\|^2] \\
& = \E[\|X - Y\|^2 \indic{T(X) = Y}] + \E[\|X - Y\|^2 \indic{T(X) \neq Y}] \\
& \leq \E[\|X - T(X)\|^2] + 4 \mathrm{d_{TV}}(T_\sharp \mu, \nu)\,,
\end{align*}
where the last inequality uses the fact that $\p[T(X) \neq Y] = \mathrm{d_{TV}}(T_\sharp \mu, \nu)$ and that $\|X - Y\| \leq 2$ almost surely.
\end{proof}

\begin{lemmapp}\label{lem:VC}
The class $\cP_{k-1}$ satisfies $\mathrm{VC}(\cP_{k-1}) \leq C dk \log k$.
\end{lemmapp}
\begin{proof}
The claim follows from two standard results in VC theory:
\begin{itemize}
\item The class all half-spaces in dimension $d$ has VC dimension $d + 1$~\citep[Corollary~13.1]{DevGyoLug96}.
\item If $\cC$ has VC dimension at most $n$, then the class $\cC_s := \{c_1 \cap \dots c_s: c_i \in \cC \, \forall i \in [s]\}$ has VC dimension at most $2ns \log(3s)$~\citep[Lemma~3.2.3]{BluEhrHau89}.
\end{itemize}
Since $\cP_{k-1}$ consists of intersections of at most $k-1$ half-spaces, we have
\begin{equation*}
\mathrm{VC}(\cP_{k-1}) \leq 3 (d+1)(k-1) \log(3(k-1)) \leq C dk \log k
\end{equation*}
for a universal constant $C$.
\end{proof}

\section{Details on numerical experiments}
\label{sec:numerics-details}

In this section we present implementation details for our numerical experiments. 

In all experiments, the relative tolerance of the objective value is used as a stopping criterion for FactoredOT. We terminate calculation when this value reached \( 10^{-6} \).

\subsection{Synthetic experiments from Section \ref{sec:hypercube}}
In the synthetic experiments, the entropy parameter was set to \( 0.1 \).

\subsection{Single cell RNA-seq batch correction experiments from Section \ref{sec:rna}}
We obtained a pair of single cell RNA-seq data sets from~\cite{haghverdi2018batch}. 
The first dataset \citep{nestorowa2016single} was generated using SMART-seq2 protocol
\citep{picelli2014full}, while the second dataset \citep{paul2015transcriptional} was generated using the MARS-seq protocol \citep{jaitin2014massively}.

We preprocessed the data using the procedure described by~\cite{haghverdi2018batch} to reduce to 3,491 dimensions.

Nex, we run our domain adaptation procedure. To determine the choice of parameters, we perform cross-validation over \( 20 \) random sub-samples of the data, each containing \( 100 \) random cells of each of the three cell types in both source and target distribution.
Performance is then determined by the mis-classification over \( 20 \) independent versions of the same kind of random sub-samples.

For all methods involving entropic regularization (FOT, OT-ER, OT-L1L2), the candidates for the entropy parameter are \( \{10^{-3}, 10^{-2.5}, 10^{-2}, 10^{-1.5}, 10^{-1}\} \).

For FOT and $k$-means OT, the number of clusters is in \( \{ 3, 6, 9, 12, 20, 30\} \).

For OT-L1L2, the regularization parameter is in \( \{10^{-3}, 10^{-2}, 10^{-1}, 1\}\).

For all subspace methods (SA, TCA), the dimensionality is in \( \{10, 20, \dots, 70\} \).

The labels are determined by first adjusting the sample and then performing a majority vote among \( 20 \) nearest neighbors.
While similar experiments \citep{PanTsaKwo11, CouFlaTui14, CouFlaTui17} employed 1NN classification because it does not require a tuning parameter, we observed highly decreased performance among all considered domain adaptation methods and therefore chose to use a slightly stronger predictor.
The results are not sensitive to the choice of \( k \) for the \( k \)NN predictor for \( k \approx 20 \).

\newpage

\end{document}